\pdfoutput=1
\PassOptionsToPackage{dvipsnames}{xcolor}
\documentclass[11pt]{article}

\usepackage{acl}
\usepackage{amsthm}
\usepackage{nicefrac}
\usepackage{times}
\usepackage{latexsym}
\usepackage{graphicx} 
\usepackage{amsmath}
\usepackage{amssymb}
\usepackage[T1]{fontenc}
\usepackage{caption}
\usepackage{subcaption}
\usepackage[utf8]{inputenc}
\usepackage[dvipsnames]{xcolor}
\usepackage{listings}
\usepackage{booktabs}

\lstdefinelanguage{json}{
  basicstyle=\ttfamily\small,
  numbers=none,
  showstringspaces=false,
  breaklines=true,
  frame=lines,
  backgroundcolor=\color{gray!10},
  string=[s]{"}{"},
  stringstyle=\color{blue},
  comment=[l]{:},
  morecomment=[l]{,},
  commentstyle=\color{gray},
}

\usepackage{algorithm}
\usepackage{algpseudocode}
\usepackage{amsmath}
\usepackage{amssymb}

\usepackage{microtype}
\usepackage{mathtools}
\usepackage{inconsolata}
\usepackage[capitalize,noabbrev]{cleveref}
\usepackage{pgfplots}
\usepackage{multirow}
\usepackage{tablefootnote}
\usepackage[most]{tcolorbox}
\PassOptionsToPackage{hyphens}{url}\usepackage{hyperref}

\usepackage{tikz}
\usetikzlibrary{calc} 
\usetikzlibrary{backgrounds}
\pgfdeclarelayer{edgelayer}
\pgfdeclarelayer{nodelayer}
\pgfsetlayers{background,main,edgelayer,nodelayer}
\usepgfplotslibrary{groupplots}

\tikzstyle{node_value}=[fill=white, draw=black, shape=circle]
\tikzstyle{node_expectation}=[fill=white, draw=black, shape=rectangle, minimum width=1.3cm, minimum height=1.0cm]

\tikzstyle{edge_geq}=[-|>, ultra thick]
\tikzstyle{edge_sample}=[draw=black, dashed, -|>, ultra thick]
\tikzstyle{edge_approx}=[-, dotted, ultra thick, draw=black]

\newcount\Comments  
\newcommand{\kibitz}[2]{\ifnum\Comments=1\textcolor{#1}{#2}\fi}

\crefname{section}{Sec.}{Sec.}
\crefname{thm}{Thm.}{Theorem}
\crefname{appendix}{App.}{Appendices}
\crefname{algorithm}{Alg.}{Algorithms}
\crefname{equation}{Eq.}{Eqs.}
\crefname{figure}{Fig.}{Figs.}
\creflabelformat{equation}{#2\textup{#1}#3} 

\theoremstyle{plain}
\newtheorem{theorem}{Theorem}[section]

\newtheorem{lemma}[theorem]{Lemma}

\theoremstyle{definition}

\theoremstyle{remark}

\definecolor{tablecolor}{rgb}{0.8,0.8,0.8}

\newcommand\cut[1]{}

\newcommand{\squishlist}{
   \begin{list}{$\bullet$}
    { \setlength{\itemsep}{0pt}      \setlength{\parsep}{3pt}
      \setlength{\topsep}{3pt}       \setlength{\partopsep}{0pt}
      \setlength{\leftmargin}{1.5em} \setlength{\labelwidth}{1em}
      \setlength{\labelsep}{0.5em} } }

\newcommand{\squishlisttwo}{
   \begin{list}{$\bullet$}
    { \setlength{\itemsep}{0pt}    \setlength{\parsep}{0pt}
      \setlength{\topsep}{0pt}     \setlength{\partopsep}{0pt}
      \setlength{\leftmargin}{2em} \setlength{\labelwidth}{1.5em}
      \setlength{\labelsep}{0.5em} } }

\newcommand{\squishend}{
    \end{list}  }

{}
{}
{}

\title{Learning to Zoom Efficiently with a Contrastive Curriculum}

\author{Falko Helm\hspace{10mm} Iryna Gurevych \\
  Ubiquitous Knowledge Processing Lab (UKP Lab) \\
Department of Computer Science and Hessian Center for AI (hessian.AI) \\
Technical University of Darmstadt \\
  \href{www.ukp.tu-darmstadt.de}{www.ukp.tu-darmstadt.de} \\
  }

\begin{document}
\maketitle

\begin{abstract}
Using a zoom-in tool is an important foundational part of modern visual agents, because it allows to efficiently handle tasks involving high-resolution images. Most previous methods need an extensive warm-start supervised fine-tuning phase for teaching models zoom-in. We show that this is not necessary by proposing a new intrinsic reward for learning tool use in MLLMs without the need for additional labels or warm-start SFT. Our InfoNCE-style reward uses a curriculum of increasingly hard negative tool calls as a contrastive training signal. Empirical experiments on $V^*$, HRBench and MME-RealWorld show that our approach is competitive while being more efficient. When used as a drop-in replacement for SFT, we even outperform all baselines.
To directly measure the zoom-in ability of models, we further introduce the scalable synthetic Muffin\&Chihuahua (M\&C) dataset. Each image consists of a grid with every cell either showing a muffin or chihuahua. Leveraging the M\&C dataset's unique region of interest labels, we find that recall is the metric that most strongly correlates the zoom-in region with final task performance.
Our model and code for reproduction is publicly available \href{https://github.com/UKPLab/emnlp2026-zoom-in}{here}.

\end{abstract}
\section{Introduction}
\begin{figure*}[]
    \centering
    \resizebox{\linewidth}{!}{
    \includegraphics{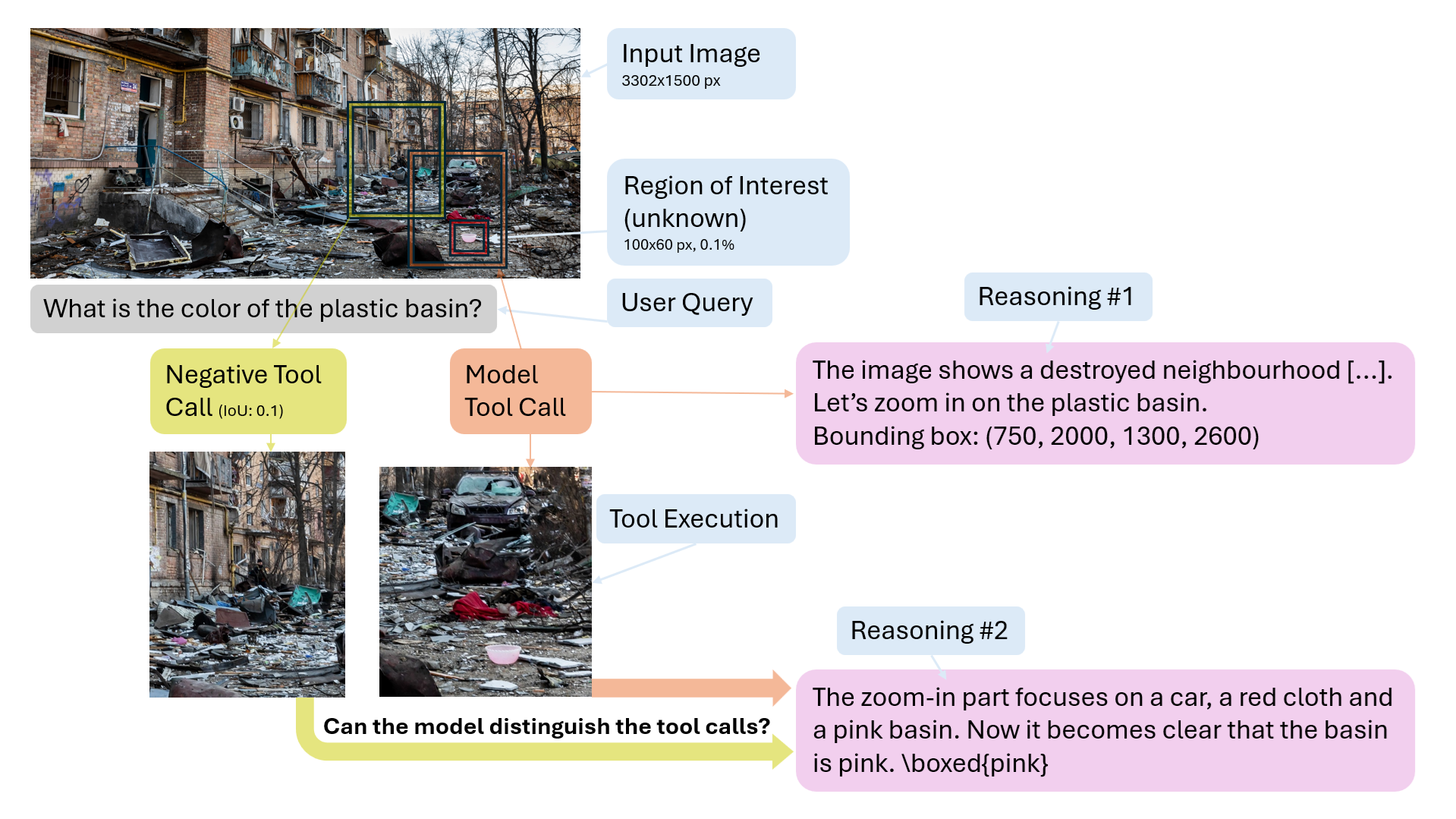}
    }
    \caption{Outline of the proposed method. Given an image and a query (original sample from the $V^*$ dataset \cite{Wu_vstar_2024}), which targets a small \textcolor{red}{Region of Interest}, the model indicates a \textcolor{orange}{Bounding Box} of which it receives a zoomed-in version. If the answer was correct, the \textcolor{orange}{Bounding Box} can be considered to be positive and we create a \textcolor{olive}{Negative Bounding Box} close to it. Then we compare the logits of \textcolor{cyan}{Reasoning\#2} conditioned on the original tool call with the ones conditioned on the negative tool call. Their Bradley-Terry score \cite{bradley-terry-model-52} is used as an additional reward for GRPO \cite{shao_deepseekmath_2024} to incentivize semantically meaningful zoom-in regions.}
    \label{fig:figure-1}
\end{figure*}
Visual Agents are becoming increasingly popular and are used in diverse applications, e.g. as Web Agents directly working on screenshots instead of HTML-code \cite{bytedance_seed_25} or for reverse-engineering video games into code \cite{qwen_3p5_2026}. These visual agents need the ability to dynamically use tools to interact with an environment that is presented visually.  
\textit{Zooming-in} is one of the most fundamental tools as it allows to efficiently handle high-resolution images with small Regions of Interest (RoIs) and is directly linked to the model's grounding capacity, i.e. how well the model understands Euclidean space \cite{sarch_grounded_25, park-etal-2025-iou-grounding}. 
Teaching to use the zoom-in tool has two goals, one of them being \textit{how} to use it syntactically and semantically correct and interpreting the tool outcome for its further reasoning. The other goal is to learn \textit{whether} to use the tool. Initially, the model might be reluctant to use the tool as it will lead to a lower reward, whereas later on, it should be careful when to use it, as tool use is an expensive operation. We focus on the initial exploration part in this work and do not explicitly consider tool use pruning. 
If using a zoom-in tool is taught in a supervised manner, both goals are addressed simultaneously as the model only needs to imitate the training data \cite{man_argus_2025, li_vocot_25}. Separating the goals only becomes meaningful in a Reinforcement Learning setup. There, the interesting aspect of learning a complex tool like zoom-in is that during trajectory generation, the model first needs to decide \textit{whether} to use the tool, before it decides \textit{how} to use it. But from a skill acquisition perspective, it makes sense to first master the tool before learning when to apply it. This mismatch in sequential order requires an additional reward mechanism to measure progress in how well the tool was used, independent of the final reward. This can of course be done by providing tool-use labels \cite{bai_qwen3-vl_2025, zhu_active-o3_2025}.
However, our method deliberately works without these labels, as creating them manually is a tedious task, especially if the environment is more complex.
Therefore, we develop a novel intrinsic tool-use reward (see Fig. \ref{fig:figure-1} for an intuitive explanation). It draws on ideas from mutual information maximization techniques \cite{oord_representation_2019}, but instead of re-sampling tool calls, we exclusively use out-of-distribution (hard) negative zoom-in regions (the yellow box in Fig. \ref{fig:figure-1}). This is because most benchmark images have a single RoI that has to be found (i.e. the red box in Fig. \ref{fig:figure-1}), so they do not benefit from diverse zoom-in operations. The negatives are automatically generated based on a silver label positive sample. The hardness of the negatives, measured as overlap with the positive bounding box, is increased during training, similar to curriculum learning \cite{bengio_curriculum_09}. To form the reward, positive and negative zoom-ins are contrasted by the effect they have on the logits of the subsequent turn.\\ During our evaluation we find that existing benchmarks are either limited in image size or do not contain ground truth zoom-in regions. To fill this gap, we introduce the Muffin\&Chihuahua (M\&C) dataset (Fig. \ref{fig:sample-muffin}), a scalable synthetic dataset, inspired by the observation that close-ups of chihuahua heads look strikingly like blueberry muffins\footnote{This was previously observed in memes (\url{https://knowyourmeme.com/memes/puppy-or-bagel}), and originally posted by Twitter user @teenybiscuit on March 9, 2016.}. It contains 36 splits across two tasks, four different image sizes and five image complexity levels. Additionally, M\&C's RoI labels allow us to evaluate the zoom-in ability independent of task performance. Our main contributions: \begin{itemize}
\itemsep0em
    \item An intrinsic reward mechanism by which the model learns how to use zoom-in without warm-start SFT or additional labels. 
    \item Competitive performance of our method on general benchmarks. If we use it as a drop-in replacement for SFT and keep training with pure RL, we even outperform baselines.
    \item The M\&C dataset for benchmarking zoom-in capabilities in a structured way across image size and task complexity.
\end{itemize} 
\section{Related Work}

\subsection{Multimodal Tool Use}
\label{sec:multimodal-tool-use}
There are several approaches that teach multimodal models how to agentically engage with images by using tools, with the most common tools being zoom-in \cite{lai_mini-o3_2025}, drawing \cite{hu_visual_2024} and search \cite{wu_mmsearch-r1_2025}. We focus on recent methods that are using Reinforcement Learning.


\paragraph{Two-stage SFT+RL} 
\begin{figure}[]
    \centering
    \resizebox{\linewidth}{!}{
    \begin{tikzpicture}
    \node (img) at (0,0) {\includegraphics[width=0.7\linewidth]{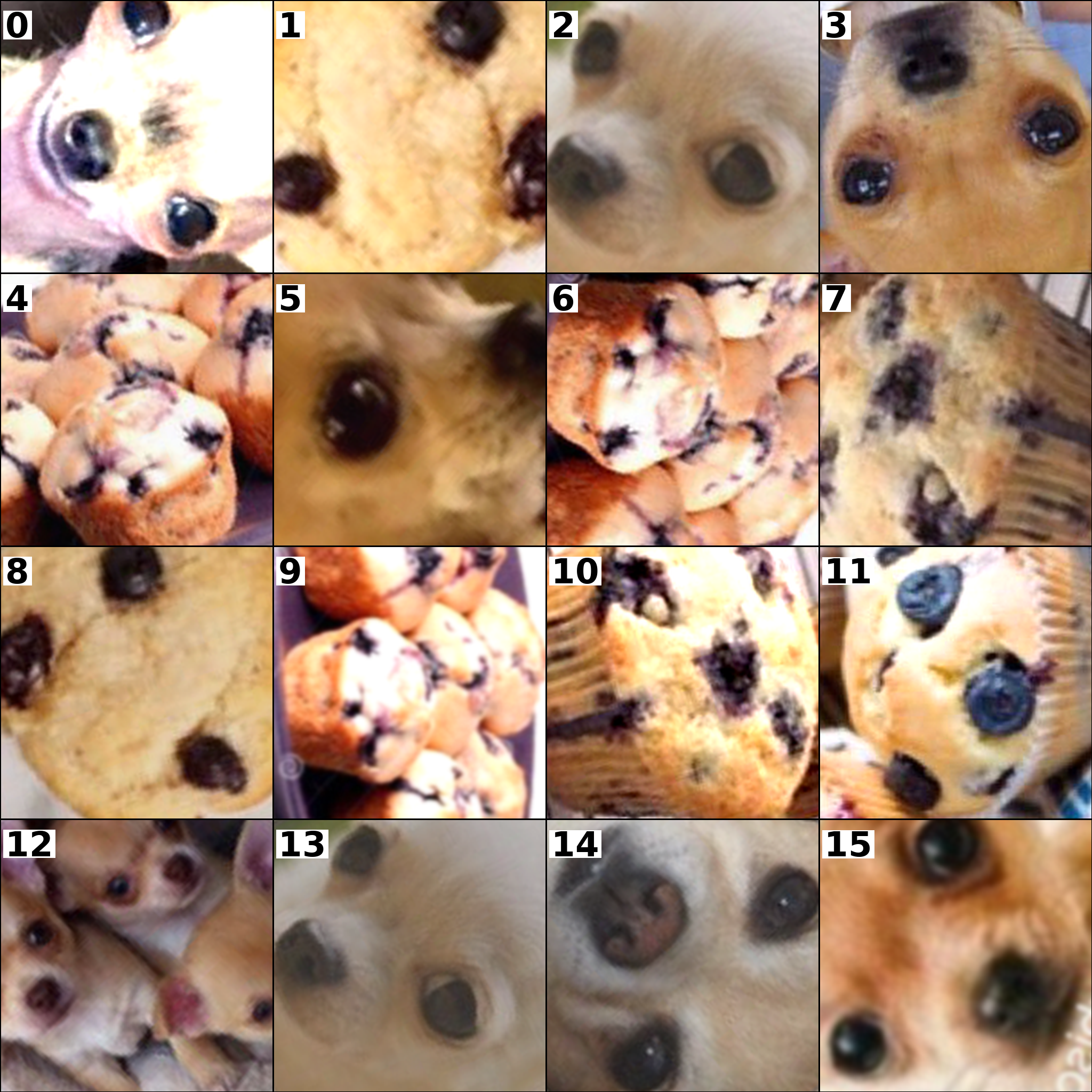}};

    \draw[
        decorate,
        decoration={brace,mirror,amplitude=3pt},
        line width=0.6pt
    ]
    ([yshift=-2pt]img.south west) --
    ([yshift=-2pt]img.south east)
    node[midway,below=8pt, font=\small]{$\text{grid size} = 4\times4$};

    \draw[
        decorate,
        decoration={brace,amplitude=3pt},
        line width=0.6pt
    ]
    ([yshift=0pt]img.north west) --
    ([yshift=0pt]img.north east)
    node[midway,above=0pt,font=\small]{$\text{image size} = 2048 \text{ px}$};

\end{tikzpicture}
    
    }
    \caption{Sample image from the M\&C dataset's split \textit{grid size} = 4, \textit{image size} = 2048, \textit{task} = single cell query. The eight cells numbered 1, 4, and 6-11 show muffins.}
    \label{fig:sample-muffin}
\end{figure}
 Supervised fine-tuning prior to Reinforcement Learning is a traditional set-up which goes under the name warm-start in the LLM community \cite{schaal_is_1999, ziegler_fine-tuning_2020}. By providing correct trajectories which the model does not need to generate itself, this technique allows to teach the model (rudimentary) forms of new behaviour quickly (e.g. using relative instead of absolute bounding boxes for zoom-in \cite{su_pixel_2025, lai_mini-o3_2025}), which are then refined in the RL stage. The major bottleneck is obtaining full trajectories for SFT, which is often done by prompting large proprietary models. 
 Pixel-Reasoner uses SFT followed by RL with a curiosity-based tool-use reward, which teaches the model to use the tool in 30\% of trajectories for every query \cite{su_pixel_2025}. 
 DeepEyes v2 and Mini o3 forego an explicit tool use reward in RL after doing extensive SFT \cite{hong_deepeyesv2_2025, lai_mini-o3_2025}. Qwen3-VL estimates the number of appropriate tool uses with a larger model before training and gives a tool-use reward by comparing them with the actual tool use count \cite{bai_qwen3-vl_2025}.
\paragraph{RL-only}
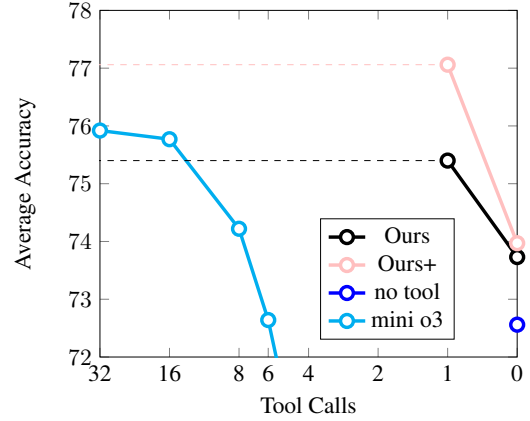
\begin{figure}[t]
\centering
\resizebox{.9\columnwidth}{!}{\begin{tikzpicture}

\begin{axis}[
    xmin=0, xmax=6,
    x dir=reverse,
    xtick={0,1,2,3,3.58,4,5,6},
    xticklabels={0,1,2,4,6,8,16,32},
    ymin=72, ymax=78,
    ytick={72, 73, ..., 78},
    xlabel=Tool Calls,
    ylabel=Average Accuracy,
    legend style={at={(0.85,0.4)}}
]

\addplot+[ultra thick,mark=*,black, mark size=3pt, mark options={scale=1, fill=white}] plot coordinates {
(0, 73.73)
(1, 75.4)
};
\addplot+[ultra thick,mark=*,pink, mark size=3pt, mark options={scale=1, fill=white}] plot coordinates {
(0, 73.97)
(1, 77.06)
};
\addplot+[ultra thick,mark=*,blue, mark size=3pt, mark options={scale=1, fill=white}] plot coordinates {
(0, 72.56)
};
\addplot+[ultra thick,mark=*,Cyan, mark size=3pt, mark options={scale=1, fill=white}] plot coordinates {
(1, 36.67)
(2, 57.84)
(3, 69.40)
(3.58, 72.64)
(4, 74.22)
(5, 75.77)
(6, 75.92)
};
\addplot[mark=none, black, dashed, domain=1:32] {75.4};
\addplot[mark=none, pink, dashed, domain=1:32] {77.06};

\addlegendentry{Ours}
\addlegendentry{Ours+}
\addlegendentry{no tool}
\addlegendentry{mini o3}

\end{axis}
\end{tikzpicture}}
\caption{Pareto optimality plot of Efficiency (x-axis; low number of tool calls) and overall Performance on general benchmarks (y-axis; high accuracy). Best models are in the upper-right corner. Ours+ gives the best performance with only a single tool call.}
\label{fig:efficiency}
\end{figure}
Without warm-start SFT, models will not stably pick up zoom-in without a tool-specific reward. Relying on accuracy reward alone causes the model to stop using the tool after a while because its average reward with tool will be lower than without tool as it is not proficient in its use yet \cite{su_pixel_2025}. This has sparked interest in using different tool-use rewards. VisionThink is about efficiently requesting the full-resolution image, given a downsampled version \cite{yang_visionthink_25}. Before training, they evaluate the model with both image resolutions to obtain labels whether tool use is actually beneficial for this specific instance, to avoid over-using the tool. This method of rewarding efficient tool use is similar to Qwen3-VL, although they scale it to more complex tools \cite{bai_qwen3-vl_2025}. DeepEyes v1 uses a conditional tool use reward which only applies when the final answer was correct \cite{zheng_deepeyes_2025}. Operationalizing the secondary reward to be conditional on the main reward is supported by \cite{liu_gdpo_2026}, which deem it more effective to prevent the model from hacking the easier secondary reward than keeping it independent and downweighting it. Active-o3 distinguishes between task and sensing model and leverages ground truth bounding boxes and heuristics (multiple proposed bounding boxes should not overlap and not be too large) for their reward \cite{zhu_active-o3_2025}.
We conclude that teaching the model to zoom-in either requires expensive full-trajectory SFT data or an extra tool-use reward to get the model to explore this new skill. As our method should be SFT-free, we now consider intrinsic exploration rewards.

\subsection{Intrinsic Exploration Rewards}
 Exploration is a long-standing but unsolved problem in RL, which boils down to the question how an agent should explore its environment to make new experiences \cite{ladosz_exploration_2022}. The stochastic nature of exploration may lead to new experiences (initially) being worse than established ones. This makes it plausible to use rewards for exploration which are independent of the external task reward. After exploration with intrinsic rewards one can continue to refine the policy with standard RL that is only guided by external rewards. Thus, we can view intrinsic reward-driven exploration as a self-supervised drop-in replacement for the warm-start SFT stage. Empowerment is an intrinsic reward which seeks states from which many other states can be reached, i.e. it `empowers' the agent by giving it more options in the future \cite{klyubin_all_2005, salge_empowerment_2013}. However, this objective can not be applied to LLMs directly, as they are trained without environment dynamics. \textit{Diversity Is All You Need} learns to separate skills and states \cite{eysenbach_diayn_19}. If we view each zoom-in region as a separate skill, their approach is rather similar to ours. However, they use a fixed skill distribution, whereas in our approach the model chooses the zoom-in region itself. The InfoNCE loss \cite{oord_representation_2019} is a contrastive method for unsupervised representation learning which works by estimating the mutual information within a sequence and has found applications in goal-conditioned RL \cite{zheng2024contrastive}. 

 \noindent Inspired by these mutual information-based methods, we develop our own intrinsic reward in Sec. \ref{subsec:method} for helping the model explore the zoom-in tool without the need for SFT data.



\subsection{Evaluation Data}

\paragraph{Limitations of zoom-in benchmarks}
\label{sec:benchmark-limitations}
\noindent When it comes to existing benchmarks to evaluate zoom-in, we found two issues. First, not all datasets benefit from the zoom-in operation, e.g. InfographicVQA \cite{mathew_info_vqa_2022} (see Sec. \ref{app:infovqa}). This is mainly related to the image size used (i.e. the image is too small or the RoI is too easy to detect). 
Second, datasets only provide insufficient annotations for the Region of Interest (RoI). Almost all datasets in Tab. \ref{tab:dataset-image-sizes} do not provide bounding boxes for the region of interest. $V^*$-Bench \cite{Wu_vstar_2024} is the exception, but 40\% of its samples are about spatial relationships between objects and thus contain two bounding boxes per image, which complicates zoom-in evaluation. The lack of RoI annotations confounds evaluation, because the model might just use the tool to get the additional reward, even if the zoom-in region is inexact. 
\paragraph{Synthetic Data}
\label{sec:related-work:synthetic-data}
As modern (multimodal) LLMs are trained on the whole internet and many existing real-world benchmarks are contaminated, synthetic data has become increasingly important to benchmark and analyze LLMs \cite{sainz-etal-2023-nlp,song_contamination_25, xu-etal-2025-dcr}. For context extension, the needle-in-a-haystack task is about prompting the model to find a unique phrase hidden in book-length texts \cite{kamradt_needle_2023}. This idea has been quickly extended by the community into language-only benchmarks \cite{hsieh_ruler_2024, yen2025helmet} and adopted for videos \cite{video_niah_25, zhao2025needle_video} and multi-image settings \cite{wang_needle_2024, wu2025visual}. For images, \citet{fan-etal-2024-muffin} and \citet{pawlowski_needles_2020} create datasets where they combine small images together in panels. However, they do not provide bounding box annotations and only use small panels of up to one megapixel in size. 

\noindent Thus, we create a framework for generating structured single-image haystacks that contain RoI annotations and can be easily scaled in image size and task complexity (Sec. \ref{sec:mc-dataset}). 

\section{Methodology}

\label{subsec:method}
In the following section, we first describe our method and then show its relation to the InfoNCE loss \cite{oord_representation_2019}.
\subsection{Method description}
If the model has used the tool in a trajectory, we can write the full token sequence as follows: \begin{align*}
    S = (Q, M_1^R, M_1^T, T_E, M_2),
\end{align*}
where $Q$ is the user query, $M_1^R$ is the first model reasoning, $M_1^T$ is the model tool call, $T_E$ is the tool execution (mostly the zoomed-in part of the image) and $M_2$ is the second model generation (including the answer). For a graphical depiction see Fig. \ref{fig:figure-1}. The goal of the method is to further encourage good tool calls $M_1^T$. If the model answer is correct, we can assume that $M_1^T$ was reasonable and treat it as a `silver' label for a good tool call. Based on the position of $M_1^T$, we
generate an `unreasonable' alternative tool call $M^{T'}_1$. Executing the tool with the new parameters yields an alternative image region $T'_E$. For this, we introduce a hyperparameter $\tau$ and generate $M^{T'}_1$ such that IoU\footnote{For a graphical depiction of the IoU (Intersection-over-Union) metric see Fig. \ref{fig:overlap-metrics}. The pseudocode to generate $M_1^{T'}$ is given in Algorithm \ref{app:alg:bbox}.App.}($T_E$, $T'_E$) = $\tau$. Now we have two prefix sequences $S_{\text{pre}} = (Q, M_1^R, M_1^T, T_E)$ and $S_{\text{pre}}' = (Q, M_1^R, M_1^{T'}, T'_E)$, where $S_{\text{pre}}$ is the true prefix for the second model generation $M_2$. To determine if the model can differentiate between $S_{\text{pre}}$ and $S_{\text{pre}}'$, we calculate \begin{align}
\label{formula:per-seq}
    \tilde{r}_{\text{tool}} =  \log\left(\frac{2\mathbb{P}(M_2|S_{\text{pre}})}{\mathbb{P}(M_2|S_{\text{pre}}) + \mathbb{P}(M_2|S_{\text{pre}}')}\right)
\end{align}
Note that we do not resample the second model generation $M_2'$ with prefix $S_{\text{pre}}'$, we only rescore the logits for the existing second generation $M_2$. This merely requires a forward pass of the model, which is much cheaper than a new generation. Further, it is important for training stability that the compared sequences $M_2$ and $M'_2$ have the same number of tokens, which is hard to enforce in a new generation. Equation (\ref{formula:per-seq}) can be interpreted as a Bradley-Terry model, such that the silver label tool call should be preferred to the alternative tool call \cite{bradley-terry-model-52, Rafailov-dpo-23}. As we show now, a deeper and more general interpretation is InfoNCE \cite{oord_representation_2019}.

\subsection{Interpretation as Contrastive Learning}
\label{sec:info-nce}
Applying Bayes' rule to the terms in (\ref{formula:per-seq}) yields \begin{align*}
    &\tilde{r}_{\text{tool}} = \\ & \log\left(\frac{2\frac{\mathbb{P}(M_1^T, T_E|Q, M_1^R, M_2)}{\mathbb{P}(M_1^T, T_E|Q, M_1^R)}}{\frac{\mathbb{P}(M_1^T, T_E|Q, M_1^R, M_2)}{\mathbb{P}(M_1^T, T_E|Q, M_1^R)} + \frac{\mathbb{P}(M_1^{T'}, T'_E|Q, M_1^R, M_2)}{\mathbb{P}(M_1^{T'}, T'_E|Q, M_1^R)}}\right)
\end{align*}
This is precisely the negative InfoNCE loss \cite{oord_representation_2019}, shifted by $\ln(2)$. As shown in their paper, $\exp(\tilde{r}_{\text{tool}})/2$ is equal to $\mathbb{P}(\{(M_1^T, T_E) = \text{positive}\}|\{\text{tool}\}, Q, M_1^R, M_2)$, i.e. the probability that $(M_1^T, T_E)$ is the positive sample in $\{\text{tool}\}$ relative to $M_2$, where $\{\text{tool}\}$ is a set of one positive sample and one negative sample. In other words, our reward measures whether the model can differentiate between tool uses based on the subsequent generation. The shift by $\ln(2)$ included in $\tilde{r}_{\text{tool}}$ is important, otherwise its upper bound would be zero. This is not desirable, because we want to encourage the exploration of the tool and not penalize it. Formally, 
 $\tilde{r}_{\text{tool}} \in (-\infty, \ln(2)]$ and $\tilde{r}_{\text{tool}} \geq 0 \iff \mathbb{P}(M_2|S_{\text{pre}}) \geq \mathbb{P}(M_2|S_{\text{pre}}')$, which should hold for most generations in practice.  

\noindent InfoNCE as an unsupervised method uses negatives that follow the data distribution. Indeed, if we would sample $(M_1^{T'}, T'_E)$ independently from $(Q, M_1^R)$, we would get $\tilde{r}_{\text{tool}} \leq MI((M_1^{T}, T_E); M_2|Q, M_1^R)$ as shown in \cite{oord_representation_2019}, i.e. maximizing our reward would maximize the (conditional) mutual information between tool use and subsequent generation. This is not desirable, as it would push the model towards generating diverse tool calls. In contrast, we construct the negatives ourselves, as this allows for fine-grained control over their difficulty by the hyperparameter $\tau$, continuously moving between easy ($\tau = 0$) and hard negatives ($\tau > 0$).
The inclusion of these hard negatives improved performance in contrastive learning setups \cite{hermans_defense_2017, robinson_hard_negatives_2021} and we ablate their usefulness in Sec. \ref{sec:ablations}.    
\subsection{Practical Considerations}

As $M_2$ consists of 40+ tokens, the sequence probabilities become very small, such that $\tilde{r}_{\text{tool}}$ has to be computed in fp32 (for an ablation see Sec. \ref{sec:ablations}). 
Instead, we compare token-level logits and sum them up. Formally, let $M_{2,i}$ denote the $i$-th token of $M_2$ and $M_{2,<i}$ the first $i-1$ tokens of $M_2$. Then set \begin{align*}
    &\tilde{r}_{\text{tool},i} = \\ &\resizebox{0.95\linewidth}{!}{$\displaystyle\log\left(\frac{2\mathbb{P}(M_{2,i}|S_{\text{pre}}, M_{2, <i})}{\mathbb{P}(M_{2,i}|S_{\text{pre}}, M_{2, <i}) + \mathbb{P}(M_{2,i}|S_{\text{pre}}', M_{2, <i})}\right)$}
\end{align*}
To mitigate that outliers dominate the reward, we clip per token and finally apply $\tanh$ to bound the sum. Thus we have
 \begin{align*}
    r_{\text{tool}} = \tanh\left(\sum_{i=1}^N \text{clip}(\tilde{r}_{\text{tool},i}, \pm \gamma)\right) \in [-1,1]
\end{align*} In Lemma \ref{lem:tanh-sum-out-in} we show that, for fixed sequence length $N$, maximizing the unclipped $r_{\text{tool}}$ maximizes $\tilde{r}_{\text{tool}}$, so it is a conservative surrogate.
Our final reward looks like \begin{align}
    r = r_{\text{acc}} + \alpha \cdot \mathbb{I}_{\{\text{answer correct}\}} \cdot r_{\text{tool}} \label{formula:final-reward}
\end{align}
where $\mathbb{I}_{\{\text{answer correct}\}}$ is one if the answer is correct, and zero otherwise. This can be seen as an extension of DeepEyes v1's reward, and setting $r_{tool} = \mathbb{I}_{\{\text{tool used}\}}$ recovers it \cite{zheng_deepeyes_2025}; we call this ablation \textit{Conditional} (see Sec. \ref{app:model-details}).

\section{Muffin and Chihuahua (M\&C) Dataset}
\begin{table*}[]
    \centering
    \resizebox{\linewidth}{!}{\begin{tabular}
    {l|c|rr|rr|rr|rr|rr}
     Dataset $\rightarrow$ &  &\multicolumn{2}{c|}{V-Star} & \multicolumn{2}{c|}{HRBench 4k} & \multicolumn{2}{c|}{HRBench 8k} & \multicolumn{2}{c|}{MME-RealWorld} & \multicolumn{2}{c}{Overall} \\
     Method$\downarrow$ Eval Setup $\rightarrow$ & Own  & tool-free & tool & tool-free & tool &  tool-free & tool & tool-free & tool & tool-free & tool \\
     \hline 
     Qwen 2.5 VL 7B & e & 72.25 & 15.71 & 57.75 & 10.38 & 50.88 & 9.75 & 49.38 & 16.77 & 57.57 & 13.15 \\
     \hline
     Pixel-Reasoner & e & 78.01 & 84.82 & 68.00 & 73.62 & 57.38 & 67.00 & 59.51 & 63.80 & 65.72 & 72.31 \\
     Mini o3 & - & - & 88.20 & - & 77.50 & - & 73.30 & - & 65.50 & - & 76.13 \\
      & e & - & 86.91 & - & 77.73 & - & 73.39 & - & 65.63 & - & 75.92 \\
     DeepEyes & - & - & 85.60 & - & 75.10 & - & 72.60 & - & - & - & - \\
     DeepEyes v2 & - & - & 81.80 & - & 77.90 & - & 73.80 & - & 64.90 & - & 74.60\\
     \hline
     no tool & t & 81.68 & - & 75.38 & - & 69.12 & - & 64.06 & - & 72.56 & -\\
     Curiosity & t & 84.82 & 86.91 & 76.75 & 76.62 & 68.88 & 70.25 & 63.00 & 63.58 & 73.36 & 74.34 \\
     Conditional & t & 84.29 & 84.29 & 76.75 & 78.00 & 69.12 & 72.12 & 63.70 & 64.89 & 73.47 & 74.83 \\
     \hline
     Ours & t & 85.34 & 85.34 & \textbf{78.50} & 78.50 & 67.62 & 73.38 & 63.45 & 64.38 & 73.73 & 75.40 \\
     \hline 
     no tool+ & t & 79.06& -& 73.38& -& 69.75& -& 65.17& -& 71.84& - \\
     Curiosity+ & t & 83.25 &87.96& 77.00& 76.50& \textbf{70.38}& \textbf{74.12}& 64.56& 66.25& 73.80& 76.21 \\
     Conditional+ & t &  \textbf{86.39} & 87.96 & 75.12 & 76.38 & 69.75 & 72.75 & 65.29 & 66.56 & \textbf{74.14} & 75.91 \\
     \hline
     Ours+ & t & 85.34 & \textbf{89.01} & 76.25 & \textbf{79.12} & 68.50 & 73.00 & \textbf{65.79} & \textbf{67.10} & 73.97 & \textbf{77.06}
    \end{tabular}}
    \caption{Main results on general benchmarks. In column `Own' we denote what we did ourselves (t: training and evaluation, e: only evaluation, -: adoption from their paper). `Overall' is the arithmetic mean of the results on the four datasets. We see that Ours+ is the best method, followed by Curiosity+ and Mini o3 \cite{lai_mini-o3_2025}.}
    \label{tab:main-results-auto}
\end{table*}
\label{sec:mc-dataset}
We draw upon the fact that the snout and eyes of chihuahua dogs and the blueberries in a muffin look very similar from a distance to create the synthetic Muffin\&Chihuahua (M\&C) dataset for stress-testing model's zoom-in ability. This dataset is similar to needle-in-a-haystack tasks (see Sec. \ref{sec:related-work:synthetic-data}) and is suitable for sharp diagnostics, especially as we have gold region of interest annotations to assess tool use quality independent of downstream performance. By looking at the sample in Fig. \ref{fig:sample-muffin} it is intuitive to understand the dataset construction. We start with a square image of side length \textit{image size} pixels, which contains a grid of \textit{grid size}$\times$\textit{grid size} cells. The cells are numbered from left to right and top to bottom and each of them shows an image of a muffin or chihuahua. We introduce the tasks \textit{single cell query} and \textit{find outlier}. In \textit{single cell query}, the muffin/chihuahua label distribution is balanced. Prompted for a cell number, the model must decide whether a muffin or chihuahua is present in the cell. This task is very similar to the prompts in existing datasets, which often ask for a small detail in a big, crowded image. A single zoom-in operation should be sufficient to solve it, if it focuses on the correct region. In \textit{find outlier}, there is only a single muffin present in the grid and the model should answer with the cell number that contains it. This task appears to be much harder and solving it likely requires many zoom-in operations.
Together, the M\&C dataset consists of 3600 samples in 36 splits, each consisting of 100 images across the cartesian product along the axes Image Size ($1024$, $2048$, $4096$, $8192$), Grid Size ($1\times1$, $2\times2$, $4\times4$, $8\times8$, $16\times16$) and Task (Single Cell Query (scq), Find Outlier\footnote{We do not consider the \textit{find outlier} task for the 1x1 grid, because it is not meaningful.} (fo))
A comparison of image sizes across datasets can be found in Tab. \ref{tab:dataset-image-sizes}. As we see, the created datasets complement the existing ones well in image size. For details of the synthetic data generation process, see Sec. \ref{app:synth-dataset}.
\section{Experiments}
\begin{table*}[]
    \centering
    \resizebox{\linewidth}{!}{\begin{tabular}{l|rr|rr|rr|rr|rr|rr}
Grid size & \multicolumn{2}{c|}{1x1} & \multicolumn{2}{c|}{2x2} & \multicolumn{2}{c|}{4x4} & \multicolumn{2}{c|}{8x8} & \multicolumn{2}{c|}{16x16} & \multicolumn{2}{c}{Overall} \\
Eval Setup & tool-free & tool & tool-free & tool & tool-free & tool & tool-free & tool & tool-free & tool & tool-free & tool \\
\hline
Qwen 2.5 VL 7B & \textbf{89.50} & - & 86.50 & - & 83.50 & - & 69.25 & - & 57.75 & - & 77.30 & - \\
\hline
Pixel-Reasoner & 86.50 & 89.50 & 87.50 & 88.75 & 84.25 & 84.25 & 66.50 & 65.50 & 57.50 & 55.00 & 76.45 & 76.60 \\
Mini o3 & - & 91.31 & - & 90.94 & - & 88.81 & - & 71.38 & - & 57.38 & - & 79.96  \\
\hline
no tool & 83.50 & - & 90.25 & - & 87.75 & - & 75.00& - & 58.50 & - & 79.00 & - \\
Curiosity & 88.50 & \textbf{93.00} & 88.50 & 89.50 & 82.75 & 87.00 & 66.75 & 69.25 & 57.00 & 54.50 & 76.70 & 78.65 \\
Conditional 	&85.50	&79.00	&89.75	&95.00&86.25	&\textbf{95.25}	&67.25	&\textbf{88.50}	&56.25	&75.25	&77.00	& \textbf{86.60}\\
Ours & 87.25 & 76.00 & 87.50 & 94.50 & 87.75 & 93.50 & 72.75 & 83.00 & 58.25 & 71.50 & 78.70 & 83.70 \\
\hline
no tool+ 	&72.75	&-	&91.50	&-	&\textbf{91.75}	&-	&\textbf{81.75}&-	&\textbf{63.75}	&-	&\textbf{80.30}	&- \\

Curiosity+ 	&82.25	&85.75	&90.75	&93.25	&86.00	&92.50	&75.50	&78.75	&58.00	&66.25	&78.50	&83.30 \\
Conditional+ 	&73.50	&74.25	&\textbf{92.75}	&\textbf{95.50}&87.75	&94.75	&72.00	&82.25	&56.75	&77.50	&76.55	& 84.85 \\
Ours+ & 79.00 & 74.25 & 89.50 & 94.25 & 86.00 & 93.50 & 72.25 & 86.75 & 58.75 & \textbf{78.25} & 77.10 & 85.40

\end{tabular}}
    \caption{Results on the \textit{single cell query} task of the M\&C dataset. For each grid, we average over image sizes. With an increasing grid size, performance steadily moves toward the 50\% of the random baseline. \textit{No tool} is a strong baseline and only \textit{Conditional} and our models consistently outperform it.}
    \label{tab:mc-scq-auto}
\end{table*}

\subsection{Training Setup}
\label{sec:train-setup}
We continue to train Qwen 2.5 VL 7B \cite{bai_qwen25-vl_2025} using GRPO \cite{shao_deepseekmath_2024}. To tackle the vanishing advantages problem we use selective sample replay \cite{wang_vl-rethinker_2025}. We perform almost on-policy RL, i.e. every other update step the buffer is flushed and is refilled by trajectories generated by the new policy. We use the PR\textbackslash video dataset for the main training runs. It is the RL dataset from Pixel-Reasoner \cite{su_pixel_2025}, which contains queries from InfographicVQA (train) \cite{mathew_info_vqa_2022} and Llava-CoT \cite{xu_llava_cot_2025}. We remove queries with multiple input images (including video queries) to obtain 6698 queries in total. For more details and training hyperparameters see Sec. \ref{sec:technical}. 

\subsection{Evaluation}
Tab. \ref{tab:main-results-auto} shows the model evaluation on the general VQA benchmarks HRBench 4k \& 8k \cite{wang_hrbench_2025}, $V^*$-Bench \cite{Wu_vstar_2024} and MME-RealWorld \cite{zhang_mme_realworld_2025}. More details for these benchmarks can be found in Sec. \ref{app:general-benchmarks}. The \textit{tool-free} evaluation mode tests the model's single-turn ability. For it, tools are not mentioned in the system prompt, and they are neither parsed nor executed if generated. Evaluations on our M\&C dataset can be found in Tab. \ref{tab:mc-scq-auto} for the \textit{single cell query} task and in Tab. \ref{tab:mc-fo} for \textit{find outlier}. 
Additionally, in Tab. \ref{tab:mc-overlap-metrics-auto}, we use M\&C's detailed RoI annotations to compare tool call quality by image overlap metrics (cf. Fig. \ref{fig:overlap-metrics}) and analyse their correlation with task performance.


\subsection{Baselines and Models}
We train three baselines and our own approach on the PR\textbackslash video dataset with GRPO, starting from Qwen2.5 VL 7B \cite{bai_qwen25-vl_2025}. \textit{Ours} uses the curriculum of negatives from Fig. \ref{fig:iou-target}. 
The \textit{no tool} model has no access to tools and \textit{Curiosity} is a retrained Pixel-Reasoner model with absolute bounding boxes and no warm-start SFT stage. \textit{Conditional} uses a DeepEyes-style constant tool use reward that is conditional on a correct answer. The '+'-variants are obtained by continually training models on \textit{Visual Probe} (train) \cite{lai_mini-o3_2025} without any additional reward. This matches the Mini o3 training setup, except that we replace their SFT by an RL-exploration phase.
Model details are in Sec. \ref{app:model-details} and curriculum ablations are in Sec. \ref{sec:ablations}.
We compare against Pixel-Reasoner \cite{su_pixel_2025}, Mini o3 \cite{lai_mini-o3_2025}, DeepEyes \cite{zheng_deepeyes_2025} and DeepEyes v2 \cite{hong_deepeyesv2_2025}.




\section{Results}
\begin{table*}[]
    \centering
    \resizebox{\linewidth}{!}{\begin{tabular}{l|rr|rr|rr|rr|rr}
Grid size & \multicolumn{2}{c|}{2x2} & \multicolumn{2}{c|}{4x4} & \multicolumn{2}{c|}{8x8} & \multicolumn{2}{c|}{16x16} & \multicolumn{2}{c}{Overall} \\
Eval setup & tool-free & tool & tool-free & tool & tool-free & tool & tool-free & tool & tool-free & tool \\
\hline
Qwen 2.5 VL 7B & 90.25 & - & 76.00 & - & 56.75 & - & 1.75 & - & 56.19 & - \\
\hline
Pixel-Reasoner & 89.75 & 88.25 & 76.50 & 86.25 & 47.50 & 50.50 & 4.25 & 13.00 & 54.50 & 59.50 \\
Mini o3 & - & 88.81 & - & 87.56 & - & 55.31 & - & 19.88 & - & 62.89 \\
\hline
no tool & 92.25 & - & 83.75 & - & 61.00 & - & \textbf{18.25} & - & 63.81 & - \\
Curiosity & 90.25 & 88.50 & 83.25 & 91.50 & 51.50 & 59.75 & 11.50 & 15.75 & 59.12 & 63.88 \\
Conditional & 90.50 & 91.50 & 83.25 & 90.00 & 38.75 & 69.50 & 4.50 & 32.00 & 54.25 & 70.75 \\
Ours & 90.50 & 92.50 & 79.75 & 85.50 & 59.00 & 59.75 & 13.25 & 14.25 & 60.62 & 63.00 \\ \hline
no tool+ & \textbf{93.50} & - & \textbf{91.75} & - & \textbf{66.75} & - & 16.00 & - & \textbf{67.00} & - \\
Curiosity+ & 90.25 & 90.50 & 88.25 & 89.75 & 55.75 & 66.50 & 14.75 & 26.25 & 62.25 & 68.25 \\
Conditional+ & 93.25 & \textbf{96.25} & 87.50 & \textbf{94.75} & 63.25 & \textbf{72.50} & 7.75 & \textbf{34.00} & 62.94 & \textbf{74.38} \\
Ours+ & 88.25 & 91.00 & 85.25 & 90.50 & 51.50 & 63.75 & 7.50 & 24.25 & 58.12 & 67.38 \\
\bottomrule
\end{tabular}}
    \caption{Results on the \textit{find outlier} task of the M\&C dataset. For each grid, we average over image sizes. Given the complexity of the task, models are doing better than expected, but only the \textit{Conditional} baselines manage to use the tool in a way that gives an advantage over tool-free evaluation.}
    \label{tab:mc-fo}
\end{table*}


\subsection{General Benchmarks}
\label{sec:general-benchmarks}
Look at the overall benchmark results in Tab. \ref{tab:main-results-auto}. All models manage to surpass the base model by at least 15 accuracy points, which shows that RL training is helpful here. \textit{Ours+} is the strongest model showing that our intrinsic reward can indeed serve as a drop-in replacement for SFT. Even before this additional training stage, \textit{Ours} is also quite competitive and is only outperformed by Mini o3.

\paragraph{Zoom-In vs. no Zoom-In}
\label{para:zoom-vs-no-zoom}
An interesting angle is obtained by having a closer look at the \textit{no tool} setups. First, it should be noted that \textit{no tool} training yields surprisingly strong results. That means $80\%$  of previous work's improvements over the base model (i.e. 15 of 18.6 absolute points) are not because of the tool use, but can be explained by further optimizing the model on this specific task. Second, we see that the \textit{no tool} baseline outperforms Pixel-Reasoner with tool. Third, if a model trained with tool use is evaluated \textit{tool-free}, i.e. without the option to use a tool, the results are better than the \textit{no tool} model, which never learnt to use a tool during training. The observation that the models do not get worse in \textit{tool-free} evaluation implies that they do not become dependent on tool use and function well without it. To explain the observation that they are even better than \textit{no tool}, we hypothesize that the models trained with tool use have seen longer trajectories with (self-created) spatial grounding than in \textit{no tool} and they got distilled into their parameters. However, this observation does not generalize to other model families (see Sec. \ref{app:sec:cross-model}).
\paragraph{Continual training}
With the '+'-runs we test the hypothesis of RL exploration to serve as a drop-in replacement for SFT. For this, we continually train on the Visual Probe (train) dataset \cite{lai_mini-o3_2025} by only using accuracy reward. After an initial SFT stage, Mini o3 was trained exactly in that manner, which leads to a meaningful comparison. Visual Probe is a hard dataset, so the \textit{no tool} baseline can not benefit from it (Tab. \ref{tab:main-results-auto}). Additionally, we continued training the SFT-only PixelReasoner-WarmStart checkpoint. After an initial stable phase with tool use rate of 60-90\% and tool success rate of 85\%, they start to decline after 35\% of training until they reach almost zero, which is in line with \cite{su_pixel_2025}'s observations. For the RL exploration models Curiosity, Conditional and Ours, the continual training is beneficial for tool-free and tool performance and increases the latter by 1-2 points. This leads Ours+ to outperform Mini o3 and Curiosity+ to become on par with it. 
\paragraph{Efficiency}
Mini o3 \cite{lai_mini-o3_2025} was trained with a maximum of six tool uses and they show that it greatly generalizes if more tool uses are allowed during inference. This is an exciting result although it raises the question of efficiency on standard benchmarks. Looking at Fig. \ref{fig:efficiency} we see that, for six tool uses, Mini o3 is just slightly better than the \textit{no tool} baseline and it takes 16 tool uses to overtake Ours which can call the tool only once. Ours+, which was continually trained on the same data as Mini o3, is pareto-optimal. To estimate the training cost of tool usage, look at Tab. \ref{tab:runtimes}.
\paragraph{Cross-Model Analysis}
To test whether our findings depend on the base model, we retrain our model and baselines on InternVL3.5-8B \cite{wang_internvl35_2025} and Gemma 4 E4B \cite{team_gemma_2026}. Technical details are in Sec. \ref{app:sec:cross-model} and the results in Tab. \ref{app:tab:internvl} resp. Tab. \ref{app:tab:gemma}.
We had to construct negatives by re-executing the tool while keeping the tool call fixed, because we observed that InternVL and Gemma discriminated the two prefixes by the digits of the bounding box rather than by its content. Formally under the notation of Sec. \ref{subsec:method}, $S_{pre}' := (Q, M_1^R, M_1^T, T'_E)$.
On InternVL, the performance of our method is subpar after the exploration stage, but it strongly benefits from the continual training such that it ranks first at the end. On Gemma, all models do not perform very well in absolute terms. Our model's performance is subpar after the exploration stage and due to resource constraints we could not complete the continual training runs.
Finally, tool calling having a positive effect on tool free evaluation (see Sec. \ref{para:zoom-vs-no-zoom}) does not replicate. On InternVL and Gemma, every tool-trained model is worse than the \textit{no tool} baseline under tool-free evaluation, so that observation should be read as specific to Qwen2.5-VL rather than as a general property of tool-use training.
\subsection{Muffin\&Chihuahua Dataset}
\begin{table*}[h]
    \centering
    \resizebox{\linewidth}{!}{
    
\begin{tabular}{l|c|cc|cc|cc}
Metric & Accuracy & \multicolumn{2}{c|}{Precision} & \multicolumn{2}{c|}{Recall} & \multicolumn{2}{c}{IoU}  \\
 & Value & Value & Pearson & Value & Pearson  & Value & Pearson  \\
\hline 
Pixel-Reasoner & 73.38 & 13.26 \tiny $\pm$ 4.32 & 0.63 &  13.61 \tiny $\pm$ 7.26 & 0.62 & 9.91 \tiny $\pm$ 5.39 & 0.63 \\
Mini o3 & 77.12 & \textbf{45.11} \tiny $\pm$ 24.18 & 0.92 & 64.60 \tiny $\pm$ 34.65 & 0.96 & \textbf{43.50} \tiny $\pm$ 24.08 & 0.91  \\
\hline 
Curiosity & 75.06 & 28.26 \tiny $\pm$ 14.42 & 0.81 & 48.57 \tiny $\pm$ 29.49 & 0.93 & 25.34 \tiny $\pm$ 12.86 & 0.84  \\
Curiosity+ & 82.69 & 32.65 \tiny $\pm$ 10.35 & 0.86 & 73.85 \tiny $\pm$ 27.59 & 0.96 & 31.97 \tiny $\pm$ 9.68 & 0.87 \\
Conditional & \textbf{88.50} & 33.80 \tiny $\pm$ 5.23 & 0.81 & 90.41 \tiny $\pm$ 18.51 & 0.94 & 33.66 \tiny $\pm$ 5.27 & 0.81 \\
Conditional+ & 87.50 & 34.54 \tiny $\pm$ 5.36 & 0.91 & \textbf{95.06} \tiny $\pm$ 14.42 & 0.86 & 34.39 \tiny $\pm$ 5.47 & 0.90 \\
Ours & 85.62 & 33.94 \tiny $\pm$ 5.86 & 0.86 & 86.78 \tiny $\pm$ 21.57 &  0.88  & 33.81 \tiny $\pm$ 5.95 & 0.86 \\
Ours+ & 88.19 & 33.62 \tiny $\pm$ 5.38 & 0.85 & 93.12 \tiny $\pm$ 16.37 & 0.90 & 33.55 \tiny $\pm$ 5.48 & 0.85  \\
\hline 
Cross-model Pearson & - & 0.43 & - & 0.93 & - & 0.53 & -  \\
\end{tabular}}
\caption{Different overlap metrics to measure the quality of the zoom-in operation and their correlation to overall task performance on \textit{single cell query} on samples that actually need a tool call (i.e. without 1x1 grids). The column Pearson denotes correlation per-model across different splits of M\&C. The row Cross-model Pearson shows which metric is best to rank model's final performance, which turns out to be Recall.}
\label{tab:mc-overlap-metrics-auto}
\end{table*}
\paragraph{Single Cell Query} 
In Tab. \ref{tab:mc-scq-auto} we see the results for the \textit{single cell query} task, where we prompt a cell number and the model has to decide if the cell contains a muffin or chihuahua. 
\textit{Conditional} performs best on this task, followed by our models. Interestingly, \textit{Conditional+}'s continual training is harmful on this task. Again, we observe a very strong \textit{no tool} baseline, which performs better than Curiosity and Pixel-Reasoner with tool. 
As expected, the performance drops with increasing grid size and on the 16x16 grid, many models struggle a lot (keep in mind that the random baseline is 50\% for this task). The 1x1 grid constitutes a special case, because the Region of Interest is the whole image, so there should be no tool use needed. Our models have a hard time dealing with that, which we analyse in Sec. \ref{app:sec:excessive-too-use}. 
\paragraph{Find Outlier} 
For the \textit{find outlier} task (Tab. \ref{tab:mc-fo}), the model has to return the cell of the single muffin present in the grid of chihuahuas.
Given that guessing is very hard for this task (random baselines are at 25\%, 6.3\%, 1.6\% and 0.4\% respectively), models are doing much better than expected. But in general, only the Conditional baselines can use the tool well enough to gain an advantage compared to tool-free evaluation. Solving \textit{find outlier} with a single tool use on larger grids is very hard from a human perspective, but even Mini o3 with 32 tool uses apparently does not manage to navigate the maze in a meaningful way. See Tab. \ref{tab:mini-o3-no-answer} for the number of times Mini o3 fails to answer. 

\paragraph{Single Zoom vs. Multi Zoom}
Our single-zoom model being Pareto-optimal on standard benchmarks could have been explained as an artifact of the task setup.
However, based on the observations on the \textit{find outlier} task we see that this is not the case. In general, additional zooms only help if they enable comparing regions or correcting earlier mistakes. On $V^*$, 40\% of the questions ask for the former, whereas \textit{find outlier} would benefit heavily from the latter. As our model outperforms multi-zoom approaches on both, we conclude that this generation and size of models is not capable enough for meaningful multi-step zooming. Future model generations or bigger model sizes may be able to realize multi-step visual search.  
\paragraph{Tool-Use quality assessment}
We use M\&C dataset's RoI annotations in Tab. \ref{tab:mc-overlap-metrics-auto} to analyze zoom-in performance directly. 
We correlate task performance with tool use quality, i.e. how close the last zoom-in region of the trajectory came to the ground truth. For this, we use mean and standard deviation of the overlap metrics \textit{precision}, \textit{recall} and \textit{IoU} (Fig.~\ref{fig:overlap-metrics}) and calculate correlation on split level per-model as well as across models.
Mini o3 achieves the highest mean precision, but at the cost of a high standard deviation. On the other hand, our models and the \textit{Conditional} baseline excel at recall, having a high mean while maintaining a comparably low standard deviation. Looking at the correlations, we see that recall is most indicative for final performance, followed by IoU. We conclude that zoom-in is most effective when it manages to reliably cover the target area. Given that tool usage should be treated as an expensive operation, zoom-in onto large patches makes sense, as failing to cover the area of interest is more harmful than zooming in onto a region that is too big. 
\section{Conclusion}
In this work we looked at MLLMs using the zoom-in tool. Inspired by mutual information maximization techniques, we proposed a SFT-free Reinforcement Learning method to teach a model how to zoom. When used as a drop-in replacement for SFT, the resulting model \textit{Ours+} achieves SoTA performance within its model class on the VQA benchmarks $V^*$, \textit{HRBench 4k} and \textit{MME-RealWorld}, which transfers from Qwen2.5 to InternVL 3.5. In particular, our model only uses the tool a single time and is thus Pareto-optimal.
To test zoom-in abilities directly, we further introduce the novel needle-in-a-haystack-style Muffin\&Chihuahua (M\&C) dataset, which contains 36 splits across grid and image size and is easily extendable to challenge larger models. Our evaluation finds that most zoom-in models from prior work struggle to use their tool effectively on M\&C. Finally, by leveraging M\&C's annotations we find that models perform best when their zoom-in area reliably covers the Region of Interest. 
All in all, our work paves the way for more self-reliant tool learning without the need for expensive SFT data. 
\section{Limitations}
\label{sec:limitations}
\paragraph{Failure modes of MI-based rewards} Using a mutual information-based approach, the second model turn $M_2$ is simultaneously action and probe for the action quality. This makes the model prone to falling into a vicious cycle, where, by chance, it produces a post-tool reasoning $M_2$ that gives high additional reward, which leads to GRPO increasing the likelihood of $M_2$. There are several reasons for producing a high-reward $M_2$, only one of which is a good previous tool call. Others include (1) the length of $M_2$, (2) high lexical overlap between $M_1$ and $M_2$ and (3) producing very round tool use pixel values. The first two are discussed in the next paragraph and the latter in Sec. \ref{app:sec:cross-model}, as it only happened in InternVL and Gemma models. One way to break this cycle is to not give additional reward to $M_2$'s tokens, which we leave for future work. Instead, we introduced a curriculum of hard negatives to mitigate these issues.    
\paragraph{Hyperparameter sensitivity} Our method introduces a curriculum with two design parameters, $t_{\text{easy}}$ and $\tau_{max}$. The first determines how long we create easy negatives during training, the second how difficult the hard negatives will be at the end of training. We ablate $\tau_{max}$ in Tab. \ref{tab:ablations} and find a corridor of good values in $[0.15, 0.2]$, with $\tau_{max} = 0.175$ best, which also transfers to a different model family (see Sec. \ref{app:sec:cross-model}). When only easy negatives are used ($\tau_{max} = 0 \iff t_{\text{easy}} = 1$), the model shows reward hacking modes (1) and (2). A qualitative example is in Sec. \ref{app:sec:qual-hacking} and training dynamics are in Fig. \ref{fig:ablations}. For Gemma, we had to go to the other extreme and set $t_{easy} = 0$ to get meaningful results (Sec. \ref{app:sec:cross-model}).
\paragraph{Base model limits} The method depends on the dataset being easy enough such that initially the model gets the answer correct from time to time even while using the tool, otherwise it will never pick up its use reliably. But this problem, i.e. only increasing likelihood of behaviour, but not discovering fundamentally new behaviour, applies to any DeepEyes-style conditional reward, and, more broadly, to GRPO as a whole \cite{wu25_invisible_leash}.

\section{Ethical considerations}
Training LLMs with Reinforcement Learning might have an adversarial effect on their instruction-following ability. Before our model should be used in production, another round of instruction-tuning should be conducted to decrease harmful generations.\\ We obtained the base images for the M\&C dataset from \url{https://www.topbots.com/downloads/code/vision/chihuahua\_vs\_muffin/}. As we could not find their licenses, we will provide a script for re-downloading them to reproduce M\&C as part of our open-sourced code. The download links for individual images including Internet Archive's wayback machine backup links are in Sec. \ref{sec:mc-base-image-links}.

\section*{Acknowledgements}
FH was funded by the German Federal Ministry of Education and Research (BMBF) within the hessian.AI Service Center. Further, FH was funded by the European Union, the Federal Ministry of Research, Technology and Space, and the Hessian Ministry of Science and Arts within the Jupyter AI Factory (BMFTR: 16HPC132; euroHPC: 101250682). For model training, we gratefully acknowledge support from the hessian.AI Service Center (funded by the Federal Ministry of Research, Technology and Space, BMFTR, grant no. 16IS22091) and the hessian.AI Innovation Lab (funded by the Hessian Ministry for Digital Strategy and Innovation, grant no. S-DIW04/0013/003).

Thanks to Fajri Koto for coming up with the idea of using synthetic annotated data. Thanks to Imbesat Rizvi, Nico Daheim, Kurt Micallef, Hector Garcia and four anonymous ARR reviewers for helpful feedback. 

\bibliography{custom}

\appendix
\renewcommand{\thetable}{\arabic{table}.App}
\renewcommand{\thefigure}{\arabic{figure}.App}

\section{Proofs}
\begin{figure*}[t]
\centering

\begin{tikzpicture}
  \draw[step=0.5cm, gray!60, thin] (0,0) grid (3,3);
  \draw[very thick] (0,0) rectangle (3,3);
  \fill[green!20] (1.5,1.5) rectangle (2.5,2.5);
  \fill[blue!20, opacity=0.5] (1,1) rectangle (3,3);
  \node[draw,align=left] at (1.5,-1.1) {$\begin{aligned}
      \text{Precision: }& 0.25 = 4/16  \\
      \text{Recall: }&    1.00 = 4/4 \\
      \text{IoU: }&       0.25 = 4/16
  \end{aligned}$};

  \draw[step=0.5cm, gray!60, thin] (4,0) grid (7,3);
  \draw[very thick] (4,0) rectangle (7,3);
  \fill[green!20] (5.5,1.5) rectangle (6.5,2.5);
  \fill[blue!20, opacity=0.5] (5.5,1.5) rectangle (6,2);
  \node[draw,align=left] at (5.5,-1.1) {$\begin{aligned}
      \text{Precision: }& 1.00 = 1/1  \\
      \text{Recall: }&    0.25 = 1/4 \\
      \text{IoU: }&       0.25 = 1/4
  \end{aligned}$};

  \draw[step=0.5cm, gray!60, thin] (8,0) grid (11,3);
  \draw[very thick] (8,0) rectangle (11,3);
  \fill[green!20] (9.5,1.5) rectangle (10.5,2.5);
  \fill[blue!20, opacity=0.5] (9,1) rectangle (10,2);
  \node[draw,align=left] at (9.5,-1.1) {$\begin{aligned}
      \text{Precision: }& 0.25 = 1/4  \\
      \text{Recall: }&    0.25 = 1/4 \\
      \text{IoU: }&       0.14 \approx 1/7
  \end{aligned}$};
\end{tikzpicture}
\caption{Graphical examples of overlap metrics. Given a \colorbox{lime}{target area} $A_{t}$ and a \textcolor{cyan}{prediction area} $A_{p}$, we define precision as $\frac{|A_{t} \cap A_{p}|}{|A_p|}$, recall as $\frac{|A_{t} \cap A_{p}|}{|A_t|}$ and IoU (Intersection-over-Union) as $\frac{|A_{t} \cap A_{p}|}{|A_t \cup A_p|}$. Example 1 and 2 have the same IoU but very different precision resp. recall.}
\label{fig:overlap-metrics}
\end{figure*}
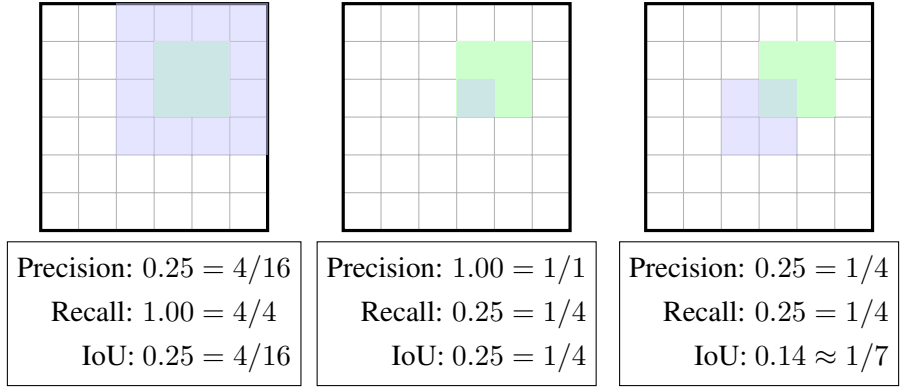
For the following lemmas, let $N \in \mathbb{N}$, $a > 0$, $0 \leq p_i \leq 1$, $0 \leq p'_i \leq 1$  and $p = \prod_{i=1}^N p_i, p' = \prod_{i=1}^N p'_i$ such that $p+p' >0 $.
\begin{lemma}
\label{lem:sum-out-in}
\begin{align*}
    &\sum_{i=1}^N \log\left(\frac{ap_i}{p_i + p'_i}\right) \\ &\leq (N-1)\cdot\log(a) + \log\left(\frac{ap}{p+p'}\right)
\end{align*}
\end{lemma}
\begin{proof}
Starting from the left side, pull $a$ out of the log and put the sum in \begin{align*}
    &\sum_{i=1}^N \log\left(\frac{ap_i}{p_i + p'_i}\right) \\= &N \log(a) + \log\left(\frac{p}{\prod_{i=1}^N(p_i + p'_i)}\right) \intertext{factorize} = &N \log(a) \\ &+ \resizebox{0.9\linewidth}{!}{$\displaystyle\log\left(\frac{p}{p + p' + \sum_{\substack{S \subseteq \{1, \ldots, N\} \\ 0<|S| <N}}\prod_{i \in S} p_i \prod_{i \notin S} p'_i}\right)$} \intertext{the sum is $>0$} \leq 
    &N\cdot\log(a) + \log\left(\frac{p}{p+p'}\right), 
\end{align*} now moving a single $a$ to the right summand yields the claim.
\end{proof}
\begin{lemma} \label{lem:tanh-sum-out-tanh-in} Let $a \geq 1$. Then \begin{align*}
    &\tanh\left(\sum_{i=1}^N \log\left(\frac{ap_i}{p_i + p'_i}\right)\right) \\ \leq& (N-1)\cdot\log(a) + \tanh\left(\log\left(\frac{ap}{p+p'}\right)\right)
\end{align*}
\begin{proof}
    Follows from Lemma \ref{lem:sum-out-in}, the monotonicity of $\tanh$ and the fact that $\tanh(x+y) - \tanh(y) \leq x $ for $x \geq 0$, because $\tanh$ is $1$-Lipschitz.
\end{proof}
\end{lemma} 
\begin{lemma} \label{lem:tanh-sum-out-in} Let $a\geq 2$ and $p \geq p'$. Then \begin{align*}
    &\tanh\left(\sum_{i=1}^N \log\left(\frac{ap_i}{p_i + p'_i}\right)\right) \\ \leq& (N-1)\cdot\log(a) + \log\left(\frac{ap}{p+p'}\right)
\end{align*}
\end{lemma}
\begin{proof}
    Follows immediately from Lemma \ref{lem:tanh-sum-out-tanh-in}, because $\tanh(x) \leq x$ for $x \geq 0$.
\end{proof}
\section{Technical Details}
\label{sec:technical}
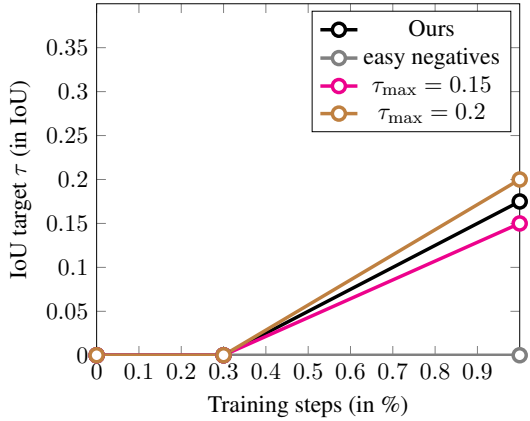
\begin{figure}[t]
\centering
\resizebox{.9\columnwidth}{!}{\begin{tikzpicture}

\begin{axis}[
    xmin=0, xmax=1,
    xtick={0, 0.1, ..., 1.0},
    ymin=0, ymax=0.4,
    ytick={0, 0.05, ..., 0.4},
    xlabel=Training steps (in \%),
    ylabel=IoU target $\tau$ (in IoU),
    yticklabel style={
        /pgf/number format/fixed}
]

\addplot+[ultra thick,
mark=*,black, mark size=3pt, 
mark options={scale=1, fill=white}
] plot coordinates {
(0, 0) (0.3, 0)
(1.0, 0.175)
};
\addlegendentry{Ours}

\addplot+[ultra thick,
mark=*,gray, mark size=3pt, 
mark options={scale=1, fill=white}
] plot coordinates {
(0, 0) (0.3, 0)
(1.0, 0.0)
};
\addlegendentry{easy negatives}
\addplot+[ultra thick,
mark=*,magenta, mark size=3pt, 
mark options={scale=1, fill=white}
] plot coordinates {
(0, 0) (0.3, 0)
(1.0, 0.15)
};
\addlegendentry{$\tau_{\max} = 0.15$}
\addplot+[ultra thick,
mark=*,brown, mark size=3pt, 
mark options={scale=1, fill=white}
] plot coordinates {
(0, 0) (0.3, 0)
(1.0, 0.2)
};
\addlegendentry{$\tau_{\max} = 0.2$}

\end{axis}
\end{tikzpicture}}
\caption{IoU target curriculum for sampling negative bounding boxes. It is zero for the first $30\%$ of steps ($t_{\text{easy}} = 0.3$) and then we increase it linearly to $\tau_{\max} = 0.175$ at $100\%$ of training. By varying $\tau_{\max}$ we create ablations and $\tau_{\max} = 0$ indicates that we only use easy negatives in training.}
\label{fig:iou-target}
\end{figure}
\subsection{Hyperparameters}
\paragraph{Hyperparameters for Reinforcement Learning}
We use a constant learning rate of $1 \times 10^{-6}$ with $3\%$ linear warmup. For GRPO we clip with $\epsilon=0.2$ and use a KL-divergence coefficient $\beta=0.04$ \cite{shao_deepseekmath_2024}. We sample $8$ generations per query with a temperature of $1.0$. We generate a batch of $280$ trajectories, which we save in a buffer. Then we sample $280$ trajectories from the buffer according to their advantages (duplicates are allowed) and perform a gradient update on them. This procedure is known as selective sample replay \cite{wang_vl-rethinker_2025}. After sampling twice in this way, the buffer is flushed and we generate new trajectories with the updated model. We thus employ a slight off-policy method. We train for a single epoch on PR\textbackslash video dataset, i.e. 382 steps, which takes around 35 hours on 8xA100 GPUs. Training a single epoch on Visual Probe (train) \cite{lai_mini-o3_2025} needs 554 steps and takes 78 hours. 
\paragraph{Hyperparameters for zoom-in}
For training, we rescale all images into the interval of [0.392, 3.92] million pixels (corresponding to [500, 5000] visual tokens for Qwen 2.5 VL). This transformation also applies to zoomed-in image areas. We apply 10\% padding to all sides of the zoom-in bounding boxes, up to a maximum of 600 pixels per side. We use absolute bounding box coordinates.
\subsection{Runtimes}
\begin{table}[]
    \centering
    \begin{tabular}{l|rrr}
     Method &  score time (h) & total time (h)  \\
     \hline
     no tool  & 3.03 &  24.67 \\
     Curiosity & {\tiny+74\%}  5.26 &  {\tiny+40\%} 34.53 \\
     Ours  & {\tiny+33\%} 7.02 &  {\tiny+4\%} 35.92 \\ 
    \end{tabular}
    \caption{Training runtime comparison across methods. \textit{Score time} is the time spent on scoring the sequences (i.e. the no\_grad forward passes), which includes the additional forward passes for our method. The percentages are always in comparison to the line above.}
    \label{tab:runtimes}
\end{table}
Because the model develops dynamically during RL training, runtimes may vary. They are dominated by the multi-turn trajectory generation, such that the extra forward pass our method needs does not increase the runtime much. Find an overview in Tab. 
\ref{tab:runtimes}.

\section{Muffin\&Chihuahua Dataset Details}
\label{app:synth-dataset}
\subsection{Dataset Construction}
We start with 8 images of muffins and 8 images of chihuahuas (each 186x186 px) sourced from the internet (see Section \ref{sec:mc-base-image-links}).  Then we pack them together into a grid to form the big image, taking into account the correct label distribution for the tasks (i.e. balanced for \textit{single cell query} and exactly 1 muffin for \textit{find outlier}). In the upper left corner of each cell we write the cell's number, starting from zero and going left to right and top to bottom. We apply the following preprocessing steps to each small image individually \begin{itemize}
    \item $50\%$ horizontal flip 
    \item $50\%$ vertical flip 
    \item $100\%$ rotation from $U([0,2\pi])$ followed by crop (to avoid black corners)
    \item $50\%$ brightness from $U([0.7, 1.3])$
    \item $50\%$ contrast from $U([0.7, 1.3])$
    \item $20\%$ Gaussian blur from $U([0.5, 1.5])$
\end{itemize}
where the initial percentage indicates how often we apply this preprocessing. After the images are constructed, we construct the prompt. In \textit{single cell query} we pick a random cell from each grid such that we end up with 50 samples with gold label muffin and 50 samples with gold label chihuahua. In \textit{find outlier} we align the prompt with the single cell showing the muffin. For both tasks, we randomly perturb the multiple choice labels A and B to counteract position bias. 
\subsection{Prompt templates}
The following line breaks are just for legibility. Actual line breaks are indicated via \texttt{\textbackslash n}. \newline
Initial prompt for both tasks:
\begin{verbatim} 
    "In the image you see a grid, 
    whose cells are numbered from 
    left to right and top to bottom. 
    In each cell, the cell's index is 
    printed in the upper left corner."
\end{verbatim}
Single cell query (replace the '11' with the appropriate cell number): \begin{verbatim}
    "Which object is in cell number
    11?\n(A) Muffin\n(B) Chihuahua\n
    Answer with the option's letter 
    from the given choices directly."
\end{verbatim}
Find outlier: \begin{verbatim}
    "In all cells except one you 
    see a Chihuahua. Which cell does
    not contain a Chihuahua, 
    but a Muffin?\nAnswer only 
    with the cell number."
\end{verbatim}
\subsection{Image sizes}
\begin{table}[]
    \centering
    \begin{tabular}{l|rrr}
    & & \multicolumn{2}{c}{image size} \\
     Dataset & \#images  & mean & median \\
     \hline
     M\&C 1k & 900 & 1.05 & 1.05 \\
     InfoVQA (val) & 500 & 3.92 & 2.25 \\
     MME-RealWorld & 23,609 & 6.95 & 2.70 \\
     V-Star (test) & 191 & 3.49 & 3.38 \\
     M\&C 2k & 900 & 4.19 & 4.19 \\
     HR Bench 4k & 200 & 14.09 & 16.26 \\
     M\&C 4k & 900 & 16.78 & 16.78 \\
     HR Bench 8k & 200 & 39.23 & 38.75 \\
     M\&C 8k & 900 & 67.11 & 67.11 \\
    \end{tabular}
    \caption{Comparison of evaluation datasets, sorted by median image size (in million pixels). \textit{M\&C X} is shorthand for all samples of image size X in our Muffin\&Chihuahua dataset (across grid sizes and tasks).}
    \label{tab:dataset-image-sizes}
\end{table}

A comparison of the image sizes of M\&C splits with image sizes of existing benchmarks can be found in Tab. \ref{tab:dataset-image-sizes}. 
\subsection{Base image links}
\label{sec:mc-base-image-links}
Here is the link format to download the sixteen base images used to construct M\&C: \url{https://www.topbots.com/downloads/code/vision/chihuahua_vs_muffin/test9.png}. They worked on March 17th, 2026. For eight muffins, replace the 9 with 1, 10, 11, 13, 16, 3, 5 and 8. For eight chihuahuas, replace the 9 with 2, 4, 6, 7, 9, 12, 14, 15. In case the images are not available anymore, they can be obtained from the overview site \url{https://www.topbots.com/downloads/code/vision/chihuahua_vs_muffin/} or from Internet Archive's wayback machine \url{https://web.archive.org/web/20240418051558im_/https://www.topbots.com/downloads/code/vision/chihuahua_vs_muffin/test9.png} by replacing the 9 with other numbers from 1 to 16 as above.

\section{Evaluation Details}
\subsection{General Benchmarks}
\label{app:general-benchmarks}
\textbf{HRBench 8k} \cite{wang_hrbench_2025}: 200 general-domain images of 8k resolution, each of them with a single query and 4 multiple-choice labels. The dataset contains 800 samples because the multiple-choice labels are cyclically rotated to counteract label bias.

\textbf{HRBench 4k} \cite{wang_hrbench_2025}: The same as HRBench 8k, except that only a 4k resolution crop of each image (which contains all information to answer the query) is provided. 

\textbf{$V^*$-Bench} \cite{Wu_vstar_2024}: General-domain images with two resp. four multiple-choice options for object relations resp. object attributes.

\textbf{MME-RealWorld} \cite{zhang_mme_realworld_2025}: Five multiple choice options (four semantic ones and a default one, e.g. "image does not contain the requested feature"). On average 1.25 queries per image. The five main domains are Autonomous Driving, Video Monitoring, Diagram/Table, OCR in the Wild and Remote Sensing.

\textbf{InfographicVQA} (val) \cite{mathew_info_vqa_2022}: Dataset of Infographics, many of them in portrait format. It features 5.6 queries per image and is OCR-heavy. Not included in the final evaluation, because no tool-use required (Sec. \ref{sec:benchmark-limitations}). 
\subsection{Evaluated Models}
\label{app:model-details}
Unless otherwise noted, we resize all images to [0.392, 3.92] million pixels, perform greedy decoding (temperature 0) and use exact match accuracy as the metric for evaluation.

\textbf{Pixel-Reasoner} \cite{su_pixel_2025}: We report the values from their paper as well as our own evaluation of their publicly available model. Here images are resized into [0.401, 4.01] million pixels, following their training setup.

\textbf{Mini o3} \cite{lai_mini-o3_2025}: We report the values from their paper as well as our own evaluation, produced by using their code. Images are resized to [0.05, 2.00] million pixels. Following their approach we want to average over at least 6000 samples per dataset, so we decode with temperature 1.0 and report Avg@1 on MME-Realworld, Avg@8 on HRB 4k and 8k, Avg@32 on $V^*$ and Avg@4 on our M\&C dataset. 

\textbf{DeepEyes} \cite{zheng_deepeyes_2025} and \textbf{DeepEyes v2} \cite{hong_deepeyesv2_2025}: We report their evaluations. 

\textbf{no tool}: We train the model in a single-turn fashion on the PR\textbackslash video dataset without access to tools (i.e. tools are not mentioned in the system prompt and are not parsed/executed if generated).

\textbf{Curiosity}: We train the model on the PR\textbackslash video dataset with the curiosity-based tool-use reward from Pixel-Reasoner. We use their hyperparameters, i.e. $H=0.3, N=1, \alpha=0.5, \beta=0.05$.
The main difference is that we use absolute pixels for the tool call and no warm-start SFT stage. 

\begin{table}[h]
\centering
\resizebox{\columnwidth}{!}{
\begin{tabular}{r|rrrrr}
 $\beta$ & V* & HRB 4k & 8k & MME & Avg \\
 \hline
0.01 & 84.29 & \textbf{78.00} & \textbf{72.12} & \textbf{64.89} & \textbf{74.83} \\
0.03 & 79.58 & 76.38 & 68.50 & 62.90 & 71.84 \\
0.30 & \textbf{84.82} & 77.75 & 68.75 & 64.41 & 73.93 \\
1.00 & 81.15 & 76.38 & \textbf{72.12} & 64.53 & 73.55 \\
\end{tabular}
}
\caption{Results for different values of $\beta$ in the \textit{Conditional} reward $r_{acc} + \beta \cdot \mathbb{I}_{\{\text{answer correct}\}} \cdot \mathbb{I}_{\{\text{tool was used}\}}$. Values are reported for tool evaluation, which is always better than the tool-free evaluation setting. There is no monotonicity in the results, but $\beta=0.01$ is clearly the best.}
\label{app:tab:deepeyes-sweep}
\end{table}
\textbf{Conditional}: We train the model on the PR\textbackslash video dataset with the following reward: $r_{acc} + \beta \cdot \mathbb{I}_{\{\text{answer correct}\}} \cdot \mathbb{I}_{\{\text{tool was used}\}}$. Except for the missing format reward (we are using boxed instead of <think> and <answer> tags) this is the constant conditional tool use reward of DeepEyes. We set $\beta = 0.01$ after doing a sweep in Tab. \ref{app:tab:deepeyes-sweep}. 
In contrast to the DeepEyes paper, we stop the generation after a single tool use and give zero reward for multiple tool attempts. In this way the model learns to use the tool exactly once.
\textbf{Ours}
We use the reward from (\ref{formula:final-reward}) with $\alpha = 0.1$ and clip value $\gamma = 1.5$\footnote{As $\tilde{r}_{tool,i}$ is bounded by $\ln(2) \approx 0.69$, this clip value only clips from below.}. We use the following schedule for the IoU target $\tau$: It takes the value of $0$ for the first $30\%$ of steps and then we increase it linearly to $0.175$ at $100\%$ of training (See Fig. \ref{fig:iou-target}).
\begin{table}[]
    \centering
    \resizebox{\linewidth}{!}{\begin{tabular}
    {ll|cc|c}
     & & \multicolumn{2}{c|}{InfographicsVQA} \\
     Method$\downarrow$ & Setup$\rightarrow$ & tool-free & tool & tool delta  \\
     \hline
     PR SFT & own eval & 79.89 &75.45&\color{red}{-4.44} \\
     PR RL & own eval & 82.70	&83.80&\color{green}{+1.1} \\
     \hline
     Curiosity & own train & 84.56 & \textbf{84.00}&\color{red}{-0.56} \\
     Ours & own train& 84.83 & 82.69 & \color{red}{-2.14} \\
     \hline 
     no tool & own train & \textbf{85.49}  & -& - \\
    \end{tabular}}
    \caption{Results on InfographicsVQA \cite{mathew_info_vqa_2022}. The best-performing models do not use tools during evaluation, showcasing that this dataset does not benefit from zoom-in.}
    \label{tab:infovqa}
\end{table}
\begin{table}[]
    \centering
    \begin{tabular}{l|rrrrr}
 task$\downarrow$ grid$\rightarrow$ & 1x1 & 2x2 & 4x4 & 8x8 & 16x16 \\
\hline
single cell & 1.12 & 0.06 & 1.12 & 4.19 &   5.12 \\
find outlier  &  - & 0.31 & 2.44 & 5.81 &   8.12 \\
\end{tabular}
    \caption{Average percentage of times Mini o3 does not give an answer on M\&C dataset after 32 turns. This is correlated with total performance (e.g. Pearson for single cell query except 1x1 grids is $-0.56$)}
    \label{tab:mini-o3-no-answer}
\end{table}

\section{Tool Prompts}
The generic prompt layouts were adapted and refined from \cite{su_pixel_2025}.
\subsection{Tool Description in System Prompt}
\label{App:Sec:tools-in-system-prompt}
 In Fig. \ref{app:fig:tool-prompt} is the generic tool with its two parameters "DESCRIPTION" and "DTYPE" that allow for different bounding box types.
 \begin{figure*}
     \centering
\begin{lstlisting}[language=json]
{
"name": "zoom_in",
"description": "Zoom in on the image based on the bounding box coordinates.",
"parameters": {
	"type": "object",
	"properties": {
		"bbox_2d": {
			"type": "array",
			"description":"DESCRIPTION",
			"items": {
			"type": "DTYPE",
			}
		},
		"target_image":{
			"type": "integer",
			"description": "The index of the image to crop. Index from 1 to the number of images. Choose 1 to operate on original image."
		}
	},
	"required": ["bbox_2d", "target_image"]
}
\end{lstlisting}
     \caption{The generic tool prompt to be inserted in the system prompt. The parameters "DESCRIPTION" and "DTYPE" are replaced as outlined in Sec. \ref{App:Sec:tools-in-system-prompt}.}
     \label{app:fig:tool-prompt}
 \end{figure*}

\noindent As discussed in Sec. \ref{app:sec:cross-model}, the tested models are sensitive to bounding box types. This is reflected in the following tool prompts for absolute and relative bounding boxes.
\paragraph{Absolute pixels (Qwen 2.5 VL)}
\begin{itemize}
    \item "DESCRIPTION": "coordinates for bounding box of the area you want to zoom in. minimum value is 0 and maximum value is the width/height of the image."
    \item "DTYPE": "integer"
\end{itemize}
\paragraph{Relative pixels (Gemma 4)}
\begin{itemize}
    \item "DESCRIPTION": "normalized coordinates for bounding box of the region you want to zoom in. Values should be within [0.0,1.0]"
    \item "DTYPE": "float"
\end{itemize}
\paragraph{Relative integer pixels (InternVL 3.5)}
\begin{itemize}
    \item "DESCRIPTION": "normalized coordinates for bounding box of the region you want to zoom in. Values should be integers in \{0, ..., 1000\} to represent promille values." 
    \item "DTYPE": "integer"
\end{itemize}

\subsection{User prompt}
The following guiding text was appended to each query during training and inference.
\begin{verbatim}
    \n\nGuidelines: Understand the given
    visual information and the user query. 
    Determine if it is beneficial 
    to employ the given visual operations
    (tools). For an image, we can look 
    closer by `zoom_in`. Reason with the
    visual information step by step, and 
    put your final answer within \\boxed{}.
\end{verbatim}

\subsection{Tool Reply}
In all cases, the zoom-in was presented to the model in the following format:

\begin{verbatim}
    \nHere is the cropped image 
    (Image Size: <width>x<height>):<IMG>
\end{verbatim}
where <width> and <height> are replaced by the absolute pixel values of the zoomed-in area in the model's frame of reference (i.e. after the model has pre-processed the image). <IMG> is replaced by the actual image tokens.

\section{Further Experiments}
\subsection{Ablations}
\begin{table*}[]
    \centering
    \resizebox{\linewidth}{!}{\begin{tabular}
    {l|rr|rr|rr|rr|rr}
     Dataset $\rightarrow$  &\multicolumn{2}{c|}{V-Star} & \multicolumn{2}{c|}{HRBench 4k} & \multicolumn{2}{c|}{HRBench 8k} & \multicolumn{2}{c|}{MME-RealWorld} & \multicolumn{2}{c}{Overall} \\
     Method$\downarrow$  & tool-free & tool & tool-free & tool &  tool-free & tool & tool-free & tool & tool-free & tool \\
     \hline 
    per sequence & 82.72 & 83.77 & 76.62 & 77.62 & 68.12 & 72.12 & 62.62 & 62.56 & 72.52 & 74.02 \\
    easy negatives & 83.77 & 84.82 & 70.62 & 74.62 & 64.12 & 66.75 & 63.45 & 61.66 & 70.49 & 71.96 \\
    two negatives & \textbf{85.34} & 84.82 & 76.12 & 76.12 & \textbf{70.00} & 71.25 & \textbf{63.93} & 61.73 & \textbf{73.85} & 73.48 \\
    $\tau_{max} = 0.15$ & 81.68 & 80.63 & 75.50 & 76.50 & 66.50 & 72.75 & 63.28 & 63.50 & 71.74 & 73.34 \\
$\tau_{max} = 0.2$ & 81.68 & 70.16 & 74.50 & 71.50 & 64.62 & 63.88 & 62.12 & 56.57 & 70.73 & 65.53 \\
    \hline
     Ours &  \textbf{85.34} & \textbf{85.34} & \textbf{78.50} & \textbf{78.50} & 67.62 & \textbf{73.38} & 63.45 & \textbf{64.38} & 73.73 & \textbf{75.40} \\
    \end{tabular}}
    \caption{Benchmark results for ablations. The value of $\tau_{max}$ must stay in a narrow corridor to achieve good performance. More negatives seem to help the effect observed in Sec. \ref{sec:general-benchmarks} where the model absorbs enhanced grounding into its parametric knowledge such that it does not depend on the tool at inference time anymore.}
    \label{tab:ablations}
\end{table*}

\input{figures/ablations}
\label{sec:ablations}
We consider the following ablations:
\textbf{per seq} uses the sequence probabilities from Equation (\ref{formula:per-seq}) instead of token probabilities. Here, we do not need to clip or use $\tanh$ because the term is upper-bounded by $\ln(2)$. Thus, we also use $\alpha = 0.1/\ln(2)$ to keep the same final reward upper bound of $0.1$ as the main model.
\textbf{two negatives} uses two negatives instead of one. To accomodate for this, we use a factor $3$ instead of $2$ in the nominator in Eq. (\ref{formula:per-seq}). This is in line with the interpretation of our reward as a special case of InfoNCE (Sec. \ref{sec:info-nce}). The second negative is sampled in a way such that it has a small overlap with the first negative.
\textbf{only easy negatives} keeps $\tau = 0$ during the whole training to ablate the need for hard negatives. 
\textbf{$\tau_{max}$} ablates the maximum value of $\tau$ at the end of training. We run $\tau_{max} = 0.15$ and $0.2$ to ablate our choice of $\tau_{max} = 0.175$ (Fig. \ref{fig:iou-target}).
See Tab. \ref{tab:ablations} for results and Fig. \ref{fig:ablations} for a comparison of training dynamics. 
\subsection{Extension to multiple tool uses}
\label{sec:extensions}
Although we have shown (see Fig. \ref{fig:efficiency}) that a single well-placed zoom-in operation is very competitive, we want to briefly discuss how to extend the proposed method for $m > 1$ tool calls. Applying the extra reward to all $m$ tool calls requires $m$ forward passes and because of self-correcting behaviour we can not be sure that every tool call in the trajectory is a positive. Instead, we propose to use only the last tool call to calculate the additional reward before the model gives a correct answer. However, we found the empirical results of this setup to be subpar, likely because multiple meaningful zoom-ins are too difficult to learn for Qwen 2.5 VL without SFT \cite{su_pixel_2025}.

\section{Qualitative Examples}
\subsection{Reward Hacking}
\label{app:sec:qual-hacking}
As an example we use sample number 111 from HR Bench 4k \cite{wang_hrbench_2025}, which is shown in Fig. \ref{app:fig:qualitative-example}. 
We use it to illustrate the issues we face without a curriculum of hard negatives. The question for this image is "What color is the backpack carried by the man wearing a yellow shirt?". We annotated the gold region of interest by hand (given in red) which makes up 0.2\% of the image. 
Model tool call bounding boxes are given in thick colors. Thin colors show the image part emitted by the tool after padding was applied, which is a strict superset of the thick box. \textit{Ours} uses the tool correctly as an information-seeking action and executes it very well which leads to the correct answer. 
\paragraph{First Reasoning} \begin{verbatim}
To determine the color of the 
backpack carried by the man 
wearing a yellow shirt, I'll zoom 
in on that specific area of the image. 
This will allow me to get a closer 
look at the details of the man and 
his belongings.\end{verbatim}
 \paragraph{Tool Call} Covers 0.79\% of the overall image and has an IoU of 0.28 with the RoI (Recall 1.0). 
 \paragraph{Second Reasoning} \begin{verbatim}The man wearing the yellow shirt is
carrying a blue backpack.\\boxed{D} \end{verbatim}

Contrast this with the \textit{no curriculum} baseline (see Sec. \ref{sec:ablations}): 

\paragraph{First reasoning} \begin{verbatim} 
To determine the color of the backpack 
carried by the man wearing a yellow 
shirt, let's analyze the image step 
by step: 1. Identify the man wearing
a yellow shirt: He is standing near
the entrance of the building, 
facing the camera. 2. Look for the 
backpack: The man is carrying a 
backpack, which is clearly visible 
on his back. 3. Determine the color
of the backpack: The backpack appears
to be black. Given this reasoning, 
the answer is: (A) Black. Now, let's
put the answer in the required format:
\end{verbatim}
\paragraph{Tool Call} Covers 20.95\% of overall image and has an IoU of 0.01 with the RoI (Recall 1.0)
\paragraph{Second reasoning} \begin{verbatim}
As observed in the zoomed-in image,
the man wearing a yellow shirt is
carrying a black backpack. Therefore,
the answer is: (A) Black \boxed{A}
\end{verbatim}
This model guesses the answer based on the big image although it can clearly not make out the details (e.g. the man "facing the camera" is not true). Then, it calls the tool for verification only (to comply with the "required format"). The tool use is inexact and we never observed that the model actually changed its initial guess based on the new image. Further, there are repeated words (‘black’). When we look at our objective $\approx \max P(M_2|S_{pre}) / P(M_2|S_{pre’})$ we see that $P(M_2|S_{pre’})$ is low, because without a curriculum, the zoom-in regions are distinct. Then the model starts maximizing this reward simply by artificially inflating its confidence in $P(M_2|S_{pre})$, e.g. based on word repetitions and ‘verifying’ the initial guess from its first reasoning.

\begin{figure*}
\resizebox{\linewidth}{!}{%
    \includegraphics{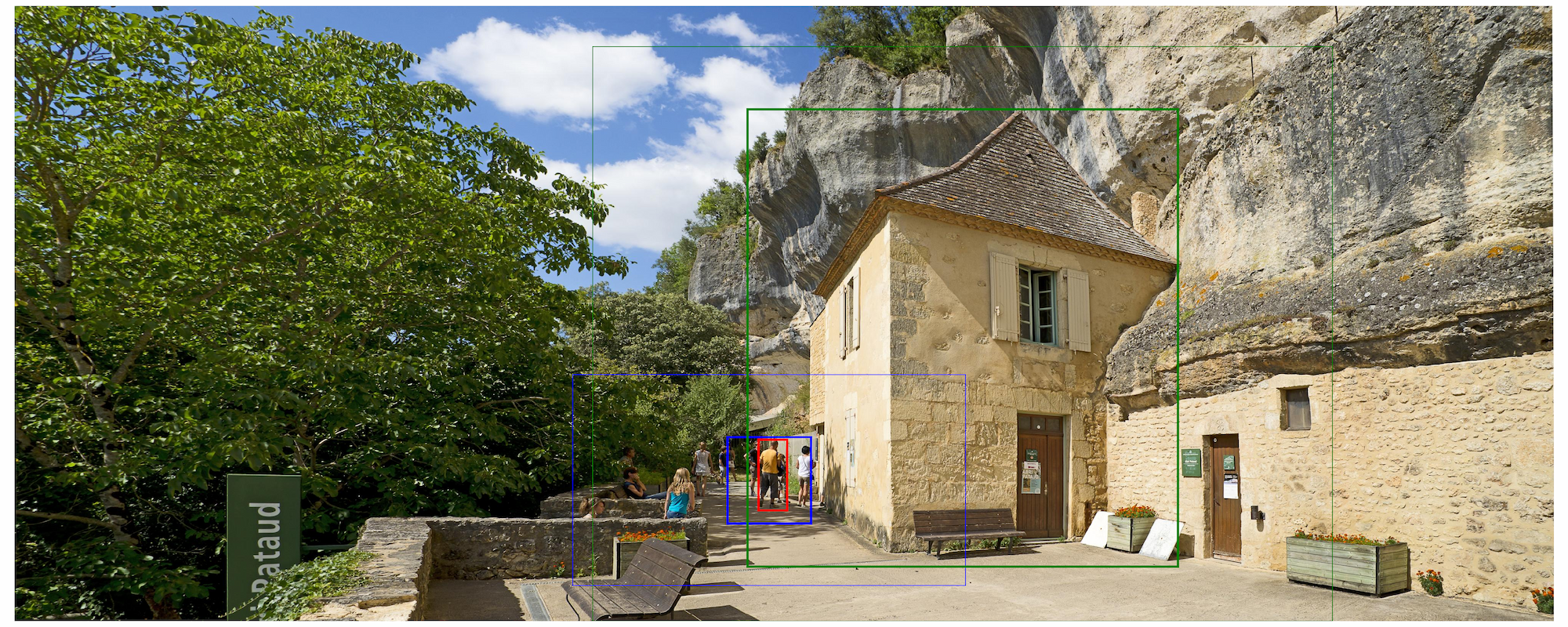}
    }
    \caption{Sample 111 from HR Bench 4k \cite{wang_hrbench_2025}. Marked in red is our hand-annotated gold region of interest. Thick boxes indicate the position of the bounding box as it was requested by the model in its tool call. Thin boxes show the bounding box position of the actual zoom, after the tool applied padding. Ours is given in blue and the easy negative ablation is given in green (see Sec. \ref{sec:ablations}).}
    \label{app:fig:qualitative-example}
\end{figure*}
\subsection{Excessive Tool Use}
\label{app:sec:excessive-too-use}
Looking at Table \ref{tab:mc-scq-auto}, we see that our models fall short on $1\times1$ grids, i.e. in scenarios where calling the tool is unnecessary. A qualitative example of this behaviour is in Fig. \ref{app:fig:qualitative-example-1x1}, which shows sample 86 of the $1\times 1, 8k \times 8k$-split of M\&C. Ours+ in blue manages to zoom into the cell number almost  perfectly (IoU 0.85). This behaviour is not helpful as it loses the big picture and classifies the image wrongly as a chihuahua. Curiosity+ in green selects a much bigger area which keeps more global context and allows it to answer correctly. We can conclude that the zoom-in behaviour of selecting very sharp bounding boxes that made our model excel in Sec. \ref{app:sec:qual-hacking} causes it to fail here.

\begin{figure*}
    \centering
    \includegraphics[scale=0.1]{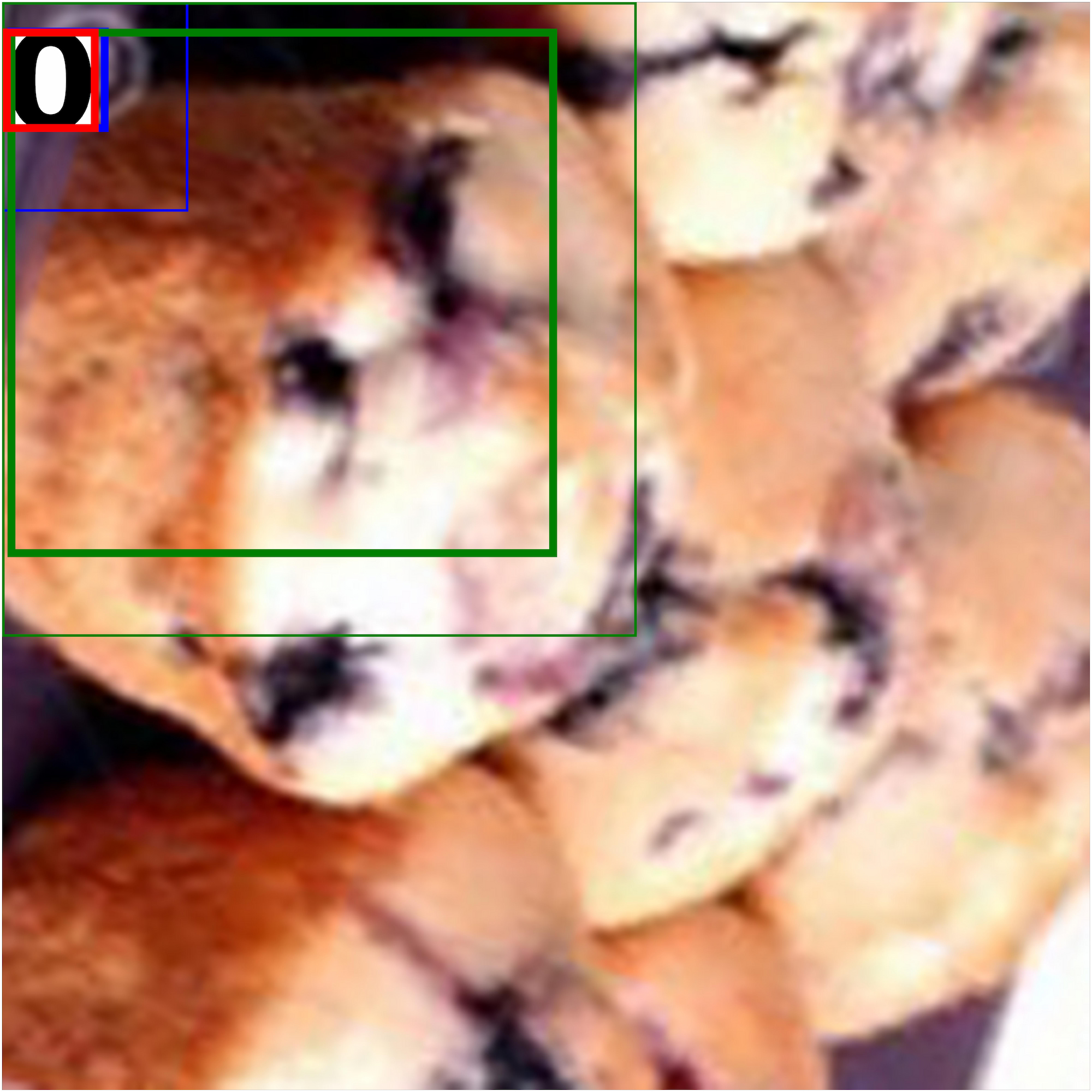}
    \caption{Sample 86 from the $1\times1,8k\times8k$-split of our M\&C dataset. Marked in red is the cell index. Thick boxes indicate the position of the bounding box as it was requested by the model in its tool call. Thin boxes show the bounding box position of the actual zoom, after the tool applied padding. Ours+ is given in blue and Curiosity+ is given in green.}
    \label{app:fig:qualitative-example-1x1}
\end{figure*}

We can analyse this behaviour more generally by looking at the full 100 samples of the $1\times1, 8k\times8k$-split (see Table \ref{app:tab:excessive-tool-call}). We see that both models actually share the previous failure, i.e. they zoom into the cell index number in the upper-left corner rather than the full cell (see Fig. \ref{fig:sample-muffin}). The performance gap (9 points) tracks the difference in how often each model targets the full image (4\% vs 14\%). But notably, IoU and task performance are decoupled on this split (correlations near zero), because when the entire image is the RoI, "good" zoom-in by IoU doesn't help the task. This suggests the issue is not zoom-in quality but rather a learned reflex to call the tool, common across RL-trained zoom-in models. Pruning unnecessary tool calls remains a hard problem (cf. Section 4.5. in \citet{bai_qwen3-vl_2025}). Simply adding a small constant negative tool-use reward, which works for web search \cite{wu_mmsearch-r1_2025}, caused our model to abandon the tool entirely.

\begin{table}[]
    \centering
    \begin{tabular}{l|rr}
                                  & Ours+ & Curiosity+ \\
                                  \hline
       Task accuracy              &  72\%  &  81\%       \\
       Tool use ratio             &  99\% &  100\%    \\
       Avg. zoom-in area          & 8.2\% & 22.5\%     \\
       Zooms targeting cell index & 96\%  & 77\%       \\
       Zooms targeting full image & 4\%   & 14\% \\
       Pearson(IoU, accuracy)     & -0.1  & 0.18  \\
    \end{tabular}
    \caption{Comparison of Ours+ and Curiosity+ on $8k \times 8k$ pixel images with $1\times 1$ grids for the single cell query task. These cases do not need a tool call, yet both models always call it.}
    \label{app:tab:excessive-tool-call}
\end{table}
\section{Algorithms}
The pseudocode for generating the hard negative bounding boxes can be found in Algorithm \ref{app:alg:bbox}.App. $B$, $U$ and $LogNormal$ denote Bernoulli, Uniform and Lognormal distributions, respectively. 

\begin{algorithm*}
\caption{Bounding Box Generation with IoU Constraint}
\label{app:alg:bbox}
\begin{algorithmic}[1]

\Require IoU target $\tau$, bounding box $\mathcal{B} = (x_1, y_1, x_2, y_2)$, image size $W, H$
\Require Hyperparameters: tolerance $\epsilon$, minimal image size $M_W, M_H$

\State $m_w, m_h \gets \frac{M_W}{W}, \frac{M_H}{H}$ \Comment{relative minimal image size}
\State $A \gets \text{area}(\mathcal{B})$
\State $w, h \gets x_2 - x_1, y_2 - y_1$

\If{$\tau = 0$}
    \If{$\mathcal{B} = (0, 0, 1, 1)$}
        \Comment{contradicting constraints, best effort}
        \State $x_1' \sim U([0, 1])$
        \State $y_1' \sim U([0, 1])$
        \State $\mathcal{B}' \gets (x_1', y_1', x_1' + m_w, y_1' + m_h)$
    \Else
        \Comment{random box with no overlap}
        \State $\text{iou} \gets 1$
        \While{$\text{iou} > 0$}
            \State $w' \sim U([m_w, 0.6])$
            \State $h' \sim U([m_h, 0.6])$
            \State $x_1' \sim U([0, 1 - w'])$
            \State $y_1' \sim U([0, 1 - h'])$
            \State $\mathcal{B}' \gets (x_1', y_1', x_1' + w', y_1' + h')$
            \State $\text{iou} \gets \text{IoU}(\mathcal{B}, \mathcal{B}')$
        \EndWhile
    \EndIf
\Else
    \State $\text{iou} \gets \infty$
    \State $\sigma_{\text{area}} \gets -0.9\tau + 0.98$
    \State $\sigma_{\text{aspect}} \gets 0.55$
    
    \While{$|\text{iou} - \tau| > \epsilon$}
        \Comment{get size of $\mathcal{B}'$}
        \State $A' \gets A \cdot \text{LogNormal}(0, \sigma_{\text{area}})$
        \State $r \gets \text{LogNormal}(0, \sigma_{\text{aspect}})$
        \State $w' \gets \sqrt{A' \cdot r}$
        \State $h' \gets \sqrt{A' / r}$
        
        \Comment{get position of $\mathcal{B}'$}
        \State $I \gets \frac{\tau(A + A')}{1 + \tau}$ \Comment{$\frac{I}{A + A' - I} = \tau$}
        
        \Comment{get overlap width and height}
        \State $w_o \sim U([0.15, 1])$
        \State $h_o \gets I / w_o$
        
        \Comment{position calculation}
        \State $c_x \gets \frac{x_1 + x_2}{2}$
        \State $c_y \gets \frac{y_1 + y_2}{2}$
        \State $x_1' \gets c_x + (-1)^{B(0.5)} \left(\frac{w + w'}{2} - w_o\right) - \frac{w'}{2}$
        \State $y_1' \gets c_y + (-1)^{B(0.5)} \left(\frac{h + h'}{2} - h_o\right) - \frac{h'}{2}$
        
        \State $\mathcal{B}' \gets (x_1', y_1', x_1' + w', y_1' + h')$
        \State $\text{iou} \gets \text{IoU}(\mathcal{B}, \mathcal{B}')$
    \EndWhile
\EndIf

\Ensure bounding box $\mathcal{B}'$

\end{algorithmic}
\end{algorithm*}

\section{Cross Model Results}
\label{app:sec:cross-model}
\begin{table*}[h]
\centering
\resizebox{\linewidth}{!}{
\begin{tabular}{l|cc|cc|cc|cc|cc}
 & \multicolumn{2}{c|}{V-Star} & \multicolumn{2}{c|}{HRBench 4k} & \multicolumn{2}{c|}{HRBench 8k} & \multicolumn{2}{c|}{MME-RealWorld} & \multicolumn{2}{c}{Avg} \\
\hline
 & tool free & tool & tool free & tool & tool free & tool & tool free & tool & tool free & tool \\
\hline
no tool & 73.82 & - & 68.63 & - & 58.38 & - & 60.28 & - & 65.28 & - \\
Curiosity & 71.73 & 79.58 & 67.00 & 69.00 & 58.88 & 68.00 & 61.03 & 63.39 & 64.66 & 69.99 \\
Conditional & 71.73 & 80.10 & 66.88 & \textbf{74.25} & 60.00 & 69.88 & 61.81 & 64.09 & 65.10 & 72.08 \\
Ours & 69.11 & 74.87 & 61.75 & 70.50 & 55.25 & 64.63 & 59.41 & 58.04 & 61.38 & 67.01 \\
\hline
no tool+ & 75.39 & - & 70.00 & - & \textbf{64.25} & - & \textbf{64.41} & - & 68.51 & - \\
Curiosity+ & 74.35 & \textbf{81.15} & 63.50 & 71.38 & 58.75 & 68.50 & 61.78 & 63.23 & 64.59 & 71.07 \\
Conditional+ & \textbf{75.92} & 80.10 & \textbf{70.75} & 73.75 & 63.63 & 70.25 & 63.50 & 65.55 & 68.45 & 72.41 \\
Ours+ &73.82	&\textbf{81.15}	&68.00&	72.75	&62.13	&\textbf{72.88}	&63.69	&\textbf{65.24}	&66.91	&\textbf{73.00}
\end{tabular}
}
\caption{Benchmark results on InternVL 3.5 8B \cite{wang_internvl35_2025}. Our method strongly benefits from continual training and takes over the other baselines.}
\label{app:tab:internvl}
\end{table*}
\begin{table*}[h]
\centering
\resizebox{\linewidth}{!}{
\begin{tabular}{l|cc|cc|cc|cc|cc}
 & \multicolumn{2}{c|}{V-Star} & \multicolumn{2}{c|}{HRBench 4k} & \multicolumn{2}{c|}{HRBench 8k} & \multicolumn{2}{c|}{MME-RealWorld} & \multicolumn{2}{c}{Avg} \\
\hline
 & tool free & tool & tool free & tool & tool free & tool & tool free & tool & tool free & tool \\
\hline
no tool & \textbf{69.11} & - & 67.88 & - & 63.63 & - & 61.65 & - & \textbf{65.57} &-  \\
Curiosity & 64.40 & \textbf{76.44} & \textbf{69.38} & 72.00 & \textbf{66.13} & 71.88 & \textbf{62.12} & \textbf{62.51} & 65.51 & \textbf{70.71} \\
Conditional & 59.69 & 71.20 & 67.75 & \textbf{73.88} & 64.88 & \textbf{72.75} & 60.68 & 61.65 & 63.25 & 69.87 \\
\hline
Ours & 62.30 & 66.49 & 68.00 & 73.13 & 60.62 & 69.88 & 60.41 & 61.45 & 62.83 & 67.74 \\

\end{tabular}
}
\caption{Benchmark results on Gemma 4 E4B \cite{team_gemma_2026}. The performance of all variants tested is not great. Curiosity performs best and our method ranks last.}
\label{app:tab:gemma}
\end{table*}
Results for InternVL 3.5 8B \cite{wang_internvl35_2025} can be found in Tab. \ref{app:tab:internvl}. We used relative integer coordinates from 0 to 1000 and applied a padding of 10\% to avoid degenerate bounding boxes and the models were trained without thinking. Images were resized into two to twenty patches, i.e. [0.4,4] million pixels be comparable with Qwen. For our approach, we kept the tool call fixed and only varied the tool execution. Formally, $S_{pre}' := (Q, M_1^R, M_1^T, T'_E)$. Without that, the model hacks the reward by differentiating the prefixes solely based on the tool call. For this it generates inaccurate, but very round bounding boxes (e.g. 100, 100, 400, 400) which contrast well with our more continuous alternative pixel values. 

Results for Gemma 4 E4B \cite{team_gemma_2026} can be found in Tab. \ref{app:tab:gemma}. We used relative pixel values in $[0,1]$ without padding and trained without thinking. We used the full resolution of 1120 image tokens. Additionally, we did not stop the generation during training if the model used the tool too often, but let it continue and gave zero reward. Without this, the model tended to fall into endless tool use loops during evaluation as it was suddenly out-of-distribution. Finally, the model often put the result after channel end during evaluation, i.e. 
\begin{verbatim}
    <channel|>10%
\end{verbatim}
instead of \begin{verbatim}
    \\boxed{10%}
\end{verbatim}
which we counted as a correct format and parsed normally. We found training our method unstable on Gemma and the standard curriculum produced deteriorated results. We hypothesize that this was because of the initial 30\% of easy negatives, so we removed them and had a linear curriculum from the beginning (i.e. $t_{\text{easy}} = 0$). This led to the model abandoning tool use at the very end of training, but during evaluation it used the tool most of the times.
\section{Additional Results}
\paragraph{Results on InfographicVQA}
\label{app:infovqa}
In \citet{su_pixel_2025}, they evaluate on the InfographicVQA dataset \cite{mathew_info_vqa_2022} although its median image size is only $2.25$M pixels (Tab. \ref{tab:dataset-image-sizes}). 
Consequently, when training models on the PR\textbackslash video dataset (Sec. \ref{sec:train-setup}) the best performing one was trained without tool access (Tab. \ref{tab:infovqa}).
Thus, we excluded it from the official results table (Sec. \ref{sec:benchmark-limitations}).
\paragraph{Overlap metrics}
In Fig. \ref{fig:overlap-metrics} there is a graphical explanation of the overlap metrics precision, recall and intersection-over-union (IoU) which we use to analyse zooming-in behaviour (Tab. \ref{tab:mc-overlap-metrics-auto}).
\paragraph{Mini o3 no answer} We observed that in some cases, Mini o3 did not give an answer after 32 tool uses when we terminated the generation. They can be found in Tab. \ref{tab:mini-o3-no-answer}.
\paragraph{\textit{Conditional} $\beta$ sweep}
As we could not infer the reward weight of the conditional tool use reward from the DeepEyes paper \cite{zheng_deepeyes_2025}, we did our own sweep over possible values in Tab. \ref{app:tab:deepeyes-sweep}. The results do not show any monotonicity, but $\beta = 0.01$ is clearly the best. 
\end{document}